\documentclass[letterpaper]{article} 
\usepackage{aaai25}  
\usepackage{times}  
\usepackage{helvet}  
\usepackage{courier}  
\usepackage[hyphens]{url}  
\usepackage{graphicx} 
\usepackage{natbib}  
\usepackage{caption} 
\usepackage{algorithm}
\usepackage{algorithmic}

\usepackage{newfloat}
\usepackage{listings}
\DeclareCaptionStyle{ruled}{labelfont=normalfont,labelsep=colon,strut=off} 
\floatstyle{ruled}
\newfloat{listing}{tb}{lst}{}
\floatname{listing}{Listing}
\usepackage{amsmath}
\usepackage{amssymb}
\usepackage{mathtools}
\usepackage{amsthm}

\usepackage[capitalize,noabbrev]{cleveref}

\newcommand{\argmin}{\mathop{\mathrm{argmin}}}

\newcommand{\encoder}{{\mathrm{encoder}}}
\newcommand{\decoder}{{\mathrm{decoder}}}

\usepackage{etoolbox}

\usepackage{subcaption}
\usepackage{amsfonts}
\usepackage{amsmath}
\usepackage{mathrsfs}
\usepackage{amsthm}
\usepackage{mathabx}
\usepackage{xcolor}
\usepackage{bbold}
\usepackage{dsfont}
\usepackage{enumitem}
\usepackage{array, booktabs}

\newcommand*{\eg}{\emph{e.g.}{}}
\newcommand*{\ie}{\emph{i.e.}{}}

\newtheorem{theorem}{Theorem}
\newtheorem*{theorem*}{Theorem}
\newtheorem{lemma}[theorem]{Lemma}

\newtheorem{remark}[theorem]{Remark}

\newtheorem{assumption}[theorem]{Assumption}

\makeatletter
\newcommand{\printfnsymbol}[1]{%
  \textsuperscript{\@fnsymbol{#1}}%
}
\makeatother

\def\E{\mathbb{E}}
\def\P{\mathbb{P}}

\def\cD{\mathcal{D}}

\def\cS{\mathcal{S}}

\def\cX{\mathcal{X}}

\def\cZ{\mathcal{Z}}

\title{Black-Box Optimization with Implicit Constraints for Public Policy}

\author{
    Wenqian Xing\textsuperscript{\rm 1}, JungHo Lee\textsuperscript{\rm 2}, Chong Liu\textsuperscript{\rm 3}, Shixiang Zhu\textsuperscript{\rm 2}
}
\affiliations {
    \textsuperscript{\rm 1}Stanford University\\
    \textsuperscript{\rm 2}Carnegie Mellon University\\
    \textsuperscript{\rm 3}University at Albany, State University of New York\\
    wxing@stanford.com, junghol@andrew.cmu.edu, cliu24@albany.edu, shixianz@andrew.cmu.edu
}

\begin{document}

\maketitle

\begin{abstract}
Black-box optimization (BBO) has become increasingly relevant for tackling complex decision-making problems, especially in public policy domains such as police redistricting. However, its broader application in public policymaking is hindered by the complexity of defining feasible regions and the high-dimensionality of decisions. This paper introduces a novel BBO framework, termed as the Conditional And Generative Black-box Optimization (\texttt{CageBO}).
This approach leverages a conditional variational autoencoder to learn the distribution of feasible decisions, enabling a two-way mapping between the original decision space and a simplified, constraint-free latent space. 
The \texttt{CageBO} efficiently handles the implicit constraints often found in public policy applications, allowing for optimization in the latent space while evaluating objectives in the original space. 
We validate our method through a case study on large-scale police redistricting problems in Atlanta, Georgia. Our results reveal that our \texttt{CageBO} offers notable improvements in performance and efficiency compared to the baselines.
\end{abstract}

\section{Introduction}
\label{sec:intro}

{In recent years, black-box optimization (BBO) has emerged as a critical approach in addressing complex decision-making challenges across various domains, particularly when dealing with objective functions that are difficult to analyze or explicitly define. Unlike traditional optimization methods that require gradient information or explicit mathematical formulations, black-box optimization treats the objective function as a ``black box'' that can be queried for function values but offers no additional information about its structure \citep{pardalos2021black}. This methodology is especially valuable for designing public policy, such as police redistricting \citep{zhu2020data, zhu2022data}, site selection for emergency service systems \citep{xing2022optimal}, hazard assessment \citep{xie2021landslide} and public healthcare policymaking \citep{chandak2020epidemiologically}. Policymakers and researchers in these areas frequently encounter optimization problems embedded within complex human systems, where decision evaluations are inherently implicit, and conducting them can be resource-intensive. Black-box optimization, therefore, stands out as a potent tool for navigating through these intricate decision spaces, providing a means to optimize outcomes without necessitating a detailed understanding of the underlying objective function's analytical properties.
}

\begin{figure}[!t]
\centering
\begin{subfigure}[b]{0.32\linewidth}
    \centering
    \includegraphics[width=\linewidth]{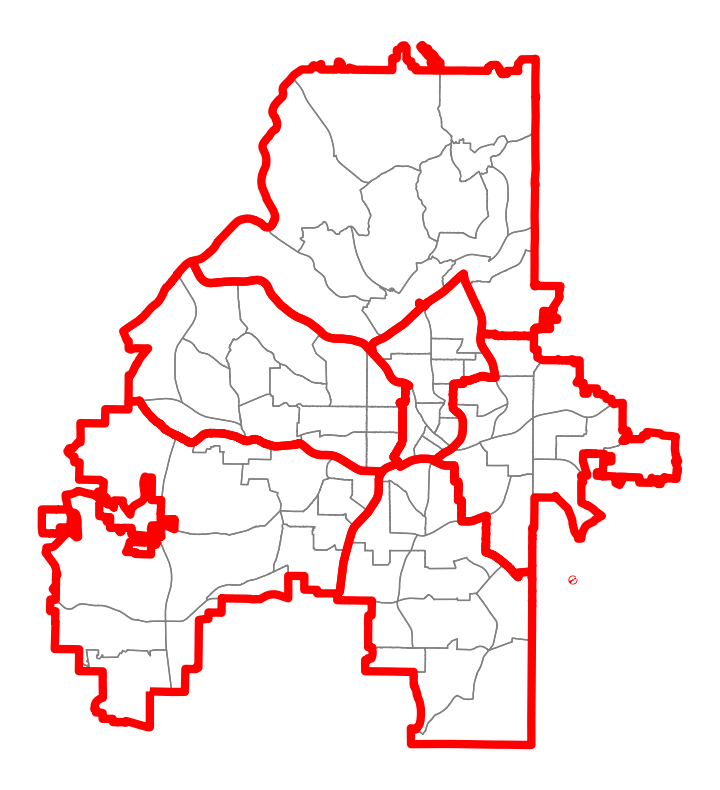}
    \caption{Pre-2019}
\end{subfigure}
\hfill
\begin{subfigure}[b]{0.32\linewidth}
    \centering
    \includegraphics[width=\linewidth]{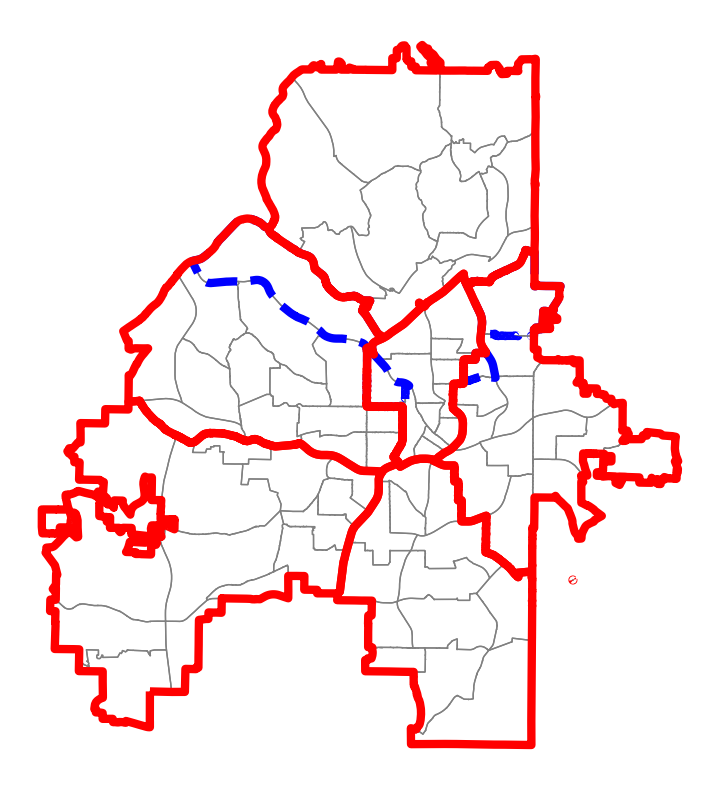}
    \caption{Post-2019}
\end{subfigure}
\hfill
\begin{subfigure}[b]{0.32\linewidth}
    \centering
    \includegraphics[width=\linewidth]{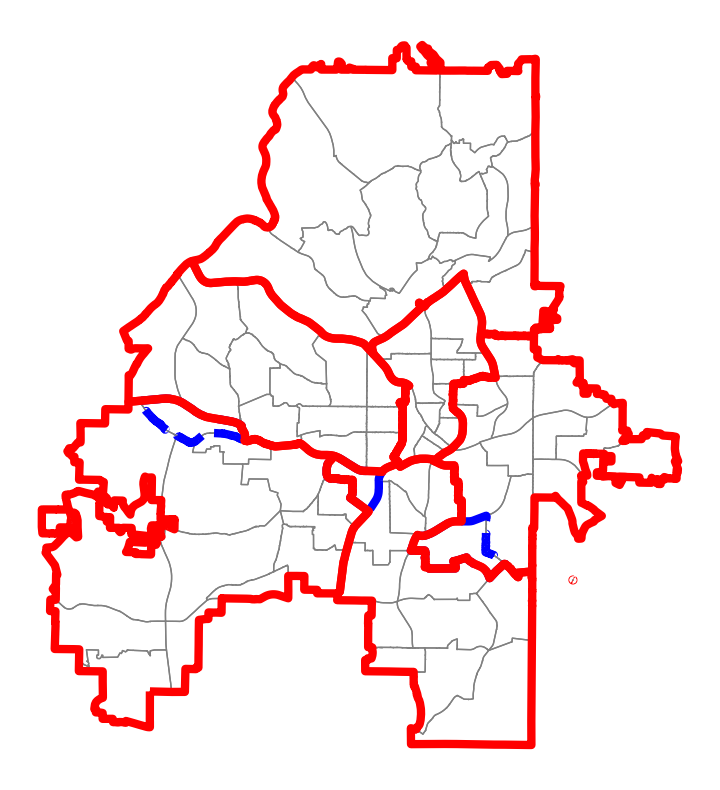}
    \caption{Infeasible}
\end{subfigure}
\caption{
An illustrative example showing the difficult-to-define constraints using three districting plans for the Atlanta Police Department (APD). Gray lines represent the basic geographical units patrolled by the police, red lines outline the districting plans, and dashed blue lines highlight the changes made to the pre-2019 plan. (a) and (b) are feasible plans implemented by the APD pre and post 2019. (c) appears to be a feasible plan but was ultimately rejected by the APD because it overlooked traffic constraints and inadvertently cut off access to some highways with its zone boundaries.
}
\label{fig:districting-exp}
\end{figure}

However, the broader application prospects of BBO in public policymaking are hampered due to two major hurdles:
(1) Defining the feasible region or setting clear constraints for decisions is inherently complicated for real human systems. 
Policymakers often encounter a myriad of both explicit and implicit rules when making an optimal decision, adding a significant layer of complexity to the optimization process. 
For instance, in police districting, police departments often organize their patrol forces by dividing the geographical region of a city into multiple patrol areas called zones. Their goal is to search for the optimal districting plan that minimizes the workload variance across zones \citep{larson1974hypercube, larson1981urban, zhu2020data, zhu2022data}. 
As shown in Figure~\ref{fig:districting-exp}, a well-conceived plan necessitates each zone to adhere to certain shape constraints (\eg, contiguity and compactness) that are analytically challenging to formulate \citep{shirabe2009districting}, while also taking socio-economic or political considerations into account (\eg, ensuring fair access to public facilities). 
This creates a web of \emph{implicit constraints} \citep{choi2000optimization, audet2020binary} that are elusive to define clearly and have a complex high-dimensional structure (e.g., manifold shape), making the assessment of feasible region nearly as expensive as evaluating the objective itself \citep{gardner2014bayesian}.
(2) The decisions for public policy are usually high-dimensional, presenting a significant computational hurdle to utilizing traditional BBO methods \citep{luong2019bayesian, wan2021think, binois2022survey}. For example, the redistricting problem can be formulated by mixed-integer programming, grappling with hundreds or even thousands of decision variables even for medium-sized service systems \citep{zhu2020data, zhu2022data}. 

Despite the difficulty in formulating constraints for decisions in public policymaking, the rich repository of historical decisions adopted by the practitioners, combined with the increasingly easier access to human systems \citep{van2017data, yu2021data}, offer a wealth of decision samples. 
Collecting these samples might involve seeking guidance from official public entities on decision feasibility or generating decisions grounded in domain expertise, bypassing the need to understand the explicit form of the constraints. 
These readily available decisions harbor implicit knowledge that adeptly captures the dynamics of implicit constraints, providing a unique opportunity to skillfully address these issues. 
This inspires us to develop a \emph{conditional generative representation model} that maps the feasible region in the original space to a lower-dimensional latent space, which encapsulates the key pattern of these implicit constraints. 
As a result, the majority of existing BBO methods can be directly applied to solve the original optimization problem in this latent space without constraints.

In this paper, we aim to solve Implicit-Constrained Black-Box Optimization (ICBBO) problems.
In these problems, \emph{the constraints are not analytically defined, however the feasibility of a given decision can be easily verified.}
We introduce a new approach called Conditional And Generative Black-box Optimization (\texttt{CageBO}).
Using a set of labeled decisions as feasible or infeasible, we first construct a conditional generative representation model based on the conditional variational autoencoder (CVAE) \citep{sohn2015learning} to explore and generate new feasible decisions within the intricate feasible region.
To mitigate the impact of potentially poorly generated decisions, we adopt a post-decoding process that aligns these decisions closer to the nearest feasible ones. 
Furthermore, we incorporate this conditional generative representation model as an additional surrogate into our black-box optimization algorithm,
allowing for a two-way mapping between the feasible region in the original space and an unconstrained latent space. 
As a result, while objective function is evaluated in the original space, the black box optimization algorithm is performed in the latent space without constraints. 
We prove that our proposed algorithm can achieve a no-regret upper bound with a judiciously chosen number of observations in the original space.
Finally, we validate our approach with numerical experiments on both synthetic and real-world datasets, including applying our method to address large-scale redistricting challenges within police operation systems in Atlanta, Georgia, demonstrating its significant empirical performance and efficiency over existing methods.

\textbf{Contributions.} Our contributions are summarized as:
\begin{enumerate}[topsep=0pt,itemsep=-1ex,partopsep=1ex,parsep=1ex]
    \item We formulate a novel class of optimization problems in public policy, termed Implicit Constrained Black-Box Optimization (ICBBO), and develop the \texttt{CageBO} algorithm that effectively tackles high-dimensional ICBBO problems. 
    \item We introduce a conditional generative representation model that constructs a lower-dimensional, constraint-free latent space, enabling BBO algorithms to efficiently explore and generate feasible candidate solutions. This model supports the development of innovative public policy decisions by facilitating the exploration of new and effective policy options.
    \item By establishing a no-regret expected cumulative regret bound, we demonstrate that our \texttt{CageBO} algorithm can consistently identify the global optimal solution, ensuring the development of near-optimal public policy decisions.
    \item We apply our \texttt{CageBO} algorithm to the Atlanta police redistricting problem using both synthetic and real datasets. Our results demonstrate its superior performance compared to baseline methods, highlighting its potential to significantly enhance efficiency in real-world public policy applications.
\end{enumerate}

\textbf{Related work.}
Black-box optimization, a.k.a. zeroth-order optimization or derivative-free optimization, is a long-standing challenging problem in optimization and machine learning. Existing work either assumes the underlying objective function is drawn from some Gaussian process \citep{williams2006gaussian} or some parametric function class \citep{dai2022sample,liu2023global}. The former one is usually known as Bayesian optimization (BO), with the Gaussian process serving as the predominant surrogate model. BO has been widely used in many applications, including but not limited to neural network hyperparameter tuning \citep{kandasamy2020tuning,turner2021bayesian}, material design \citep{ueno2016combo,zhang2020bayesian}, chemical reactions \citep{guo2023bayesian}, and public policy \citep{xing2022optimal}.

In numerous real-world problems, optimization is subject to various types of constraints. 
Eriksson proposes scalable BO with known constraints in high dimensions \citep{eriksson2021scalable}. Letham explores BO in experiments featuring noisy constraints \citep{letham2019constrained}. Gelbart pioneered the concept of BO with unknown constraints \citep{gelbart2014bayesian}, later enhanced by Aria through the ADMM framework \citep{ariafar2019admmbo}. Their constraints are unknown due to uncertainty but can be evaluated using probabilistic models. In addition, Choi and Audet study the unrelaxable hidden constraints in a similar way \citep{choi2000optimization, audet2020binary}, where the feasibility of a decision can be evaluated by another black-box function. In contrast, the implicit constraints are often unknown due to the lack of analytical formulations in many public policy-making problems.

Building on latent space methodologies, Varol presented a constrained latent variable model integrating prior knowledge \citep{varol2012constrained}. Eissman \citep{eissman2018bayesian} presents a VAE-guided Bayesian optimization algorithm with attribute adjustment. Deshwal and Doppa focus on combining latent space and structured kernels over combinatorial spaces \citep{deshwal2021combining}. Maus further investigates structured inputs in local spaces \citep{maus2022local}, and Antonova introduces dynamic compression within variational contexts \citep{antonova2020bayesian}. However, it's worth noting that none of these studies consider any types of constraints in their methodologies.

In public policy-making, implicit constraints are often a critical consideration in queueing service models. The pioneering work on the hypercube queueing model, introduced in Larson's paper \citep{larson1974hypercube} and further developed in his book \citep{larson1981urban}, established a foundational framework for analyzing spatial queueing systems, particularly in the context of emergency services like police and ambulance operations. This model is frequently employed as a black-box performance measure for public resource deployment \citep{de2015incorporating, ansari2017maximum, zhu2020data}. However, real-world service systems are often subject to various constraints on feasible policies. 
These constraints, such as fairness \citep{argyris2022fair} and continuity \citep{zhu2022data}, are typically implicit and lack explicit analytical formulations, posing challenges for effective optimization.

\section{Preliminaries}

\textbf{Problem setup.}
We consider a decision space denoted by $\mathcal{X} \subseteq [0, 1]^d$ where $1$ can be replaced with any universal constant w.l.o.g., which represents a specific region of a $d$-dimensional real space. Suppose there exists a black-box objective function, $f: \mathcal{X} \mapsto \mathbb{R}$, that can be evaluated, albeit at a substantial cost. Assume we can obtain a noisy observation of $f(x)$, denoted as $\hat{f}(x) = f(x) + \epsilon$, where $\epsilon$ follows a $\sigma$-sub-Gaussian noise distribution. The goal is to solve the following optimization problem:
\begin{align*}
\min_{x \in \cX} ~ f(x),\quad \text{s.t.} ~ x \in \cS,
\end{align*}
where \( \mathcal{S} \subseteq \mathcal{X} \) represents the feasible region, defined by a set of implicit constraints.

{ Given the analytical expressions of the implicit constraints are not directly accessible, explicitly formulating these constraints is not feasible. However, they can still be evaluated through a feasibility oracle $h(\cdot): \mathcal{X} \mapsto \{0, 1\}$, where a value of 1 indicates feasible, and 0 indicates infeasible. Now suppose we have access to a human system that provides labeled decisions. Denote a set of labeled decisions by $\mathcal{D} = \{(x_i, c_i)\}, i \in [n]$, where $x_i \in \mathcal{X}$ represents the $i$th decision and $c_i \in \{0, 1\}$ represents its feasibility. Assume $\mathcal{D}$ has a good coverage of the decision space of interest.
In practice, decisions can be derived from consultations within human systems or crafted using domain expertise, with the feasibility oracle effectively acting as a surrogate for a policymaker.
For example, new feasible districting plans in Figure~\ref{fig:districting-exp} can be created by first randomly altering the assignments of border regions and then checking their feasibility through police consultations.}

\textbf{Bayesian optimization.}
The BO algorithms prove especially valuable in scenarios where the evaluation of the objective function is costly or time-consuming, or when the gradient is unavailable. This approach revolves around constructing a \emph{surrogate model} of the objective function and subsequently employing an \emph{acquisition function} based on this surrogate model to determine the next solution for evaluation. For the minimization problem, a popular choice of surrogate model is the Gaussian process with the lower confidence bound (LCB) \citep{srinivas2009gaussian} serving as the acquisition function. 

The Gaussian process (GP) in the space $\mathcal{X}$, denoted by $\text{GP}({\mu}, {k}; \mathcal{X})$, is specified by a mean function ${\mu(x)}$ and a kernel function $k(x, x^\prime)$, which indicates the covariance between the two arbitrary decisions $x$ and $x^\prime$. 
The GP captures the joint distribution of all the evaluated decisions and their observed objective function values. 
We reference the standard normal distribution with zero mean and an identity matrix $I$
as its variance by $\mathcal{N}(0, I)$,
and let $Y = \{\hat{f}(x) \mid x \in X \}$ represent the corresponding set of objective function values.
For a new decision $\tilde{x}$, the joint distribution of $Y$ and its objective function value $\tilde{y}$ of $\tilde{x}$ is
\begin{equation*}
\begin{bmatrix}
Y \\
\tilde{y}
\end{bmatrix} \sim \mathcal{N}\left(\mu\left(\begin{bmatrix}
X \\
\tilde{x}
\end{bmatrix}\right), \begin{bmatrix}
K(X, X)+\sigma^2 I & K(X, \tilde{x}) \\
K(X, \tilde{x})^\top & k(\tilde{x}, \tilde{x})
\end{bmatrix}\right),
\end{equation*}
where $\sigma^2$ represent the variance of the observed noise $\epsilon$. 
Here $K(X, X)=\left ( k\left(x, x^{\prime}\right) \right )_{x, x^{\prime} \in \mathcal{X}}$ denotes the covariance matrix between the previously evaluated decisions and $K(X, \tilde{x})=\left (  k\left(x, \tilde{x}\right) \right )_{x \in \mathcal{X}}$ denotes the covariance vector between the previously evaluated and the new decisions.

\section{Proposed Method}

\begin{figure}
\begin{center}
  \includegraphics[width=1\linewidth]{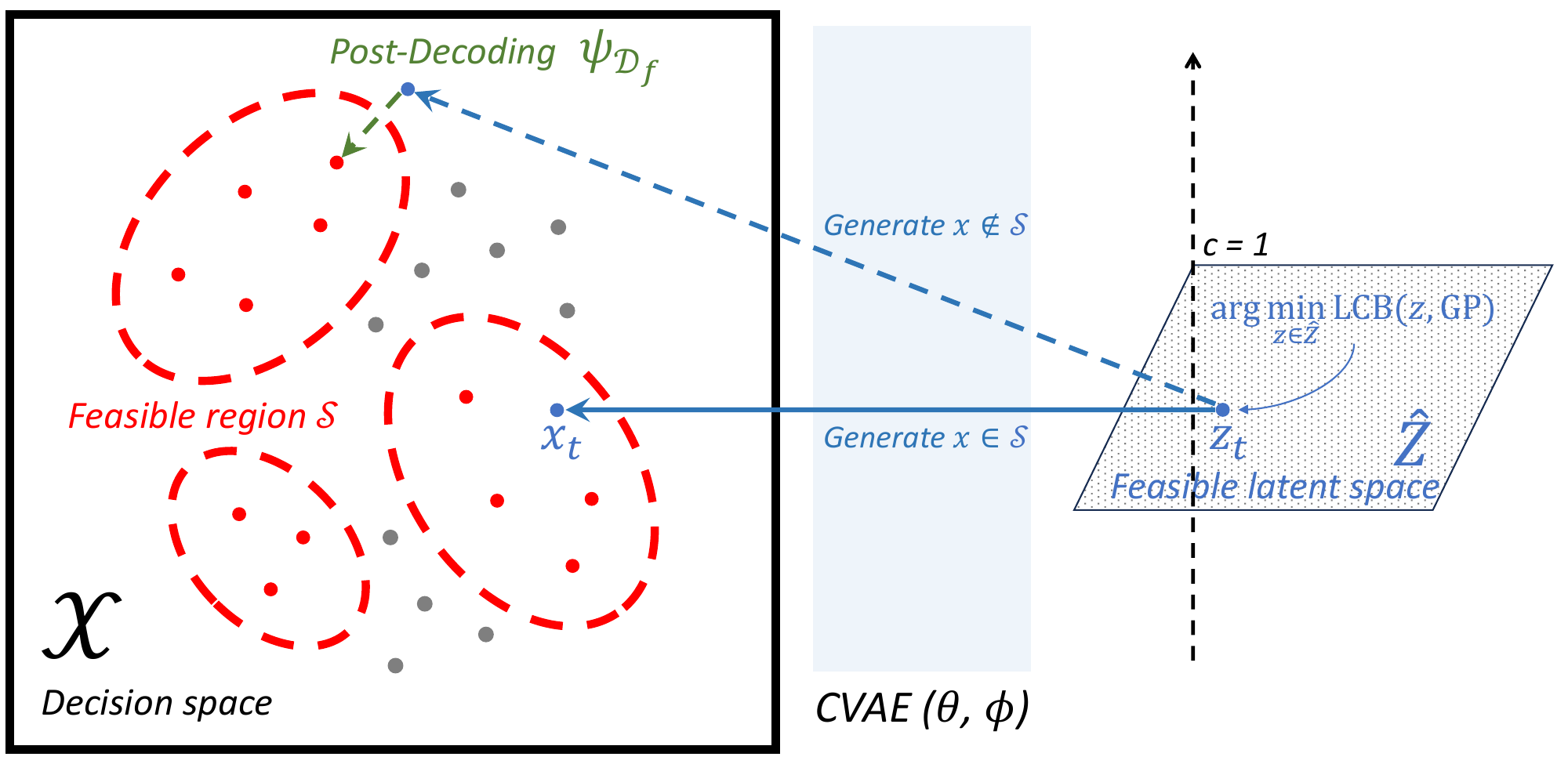}
\end{center}
  \caption{An illustration of the \texttt{CageBO} algorithm. Red dots represent observed feasible decisions, and grey dots denote observed infeasible decisions. The red dotted circle illustrates the complex feasible region that is not directly accessible. For each iteration, a new decision $z_{t}$ is then chosen in the feasible latent space $\hat{Z}$ by minimizing the lower confidence bound (LCB). This new decision $z_{t}$ is mapped back to the original space as $x_t$ using $\decoder_{\theta}$. If the CVAE model is well-trained, the newly generated \(x_t\) is highly likely to reside within the feasible set $\cS$. In case \(x_t \notin \cS \), the post-decoding process \(\psi_{\mathcal{D}_f}\) will adjust it to the nearest observed feasible decision.}
  \label{fig:illustration}
\end{figure}

The main idea of our method is to learn a conditional generative representation of feasible decisions to effectively overcome the complications posed by implicit constraints. Then we perform Bayesian optimization within this low-dimensional representation space rather than the constrained original decision space. The latent space $\mathcal{Z} \subseteq \mathbb{R}^{d}$ is learned using a conditional variational autoencoder (CVAE), which leverages the set of labeled decisions $\mathcal{D}$ as the training data.

To be specific, we train the CVAE model by maximizing $\mathcal{L}_{\mathrm{ELBO}}$ on the training data $\mathcal{D}$. This CVAE model enables a two-way mapping of decisions between the original and latent spaces, through an \emph{encoder$_{\phi}$} from the learned conditional distribution \( q_{\phi}(z | x, c) \) and a \emph{decoder$_{\theta}$} from the learned conditional posterior \( p_{\theta}(x|z, c) \).
Given an initial set of decisions $X_0$ randomly sampled from the feasible set $\mathcal{D}_f$, we first evaluate their objective values denoted as $Y_0 = \{\hat{f}(x) \mid x \in X_0\}$.
These decisions are then encoded to the latent space, represented by $Z_0 \subset \mathcal{Z}$. 
As illustrated by Figure~\ref{fig:illustration}, for each iteration $t$, our algorithm is performed as follows:
(1) Train a surrogate model $\text{GP}(\mu, k;\mathcal{Z})$ using the current latent decisions $Z_{t-1}$ and their observed values $Y_{t-1}$. 
(2) Identify the next latent decision candidate \( z_{t} \) by searching on the feasible latent space $\hat{Z}$, defined as the support of the latent variable distribution conditioned on feasibility $c=1$, and select the one exhibiting the lowest LCB value.
(3) Decode the latent decision candidate $z_{t}$ to the original space, yielding a new decision ${x}_{t}$.
(4) Assess the feasibility of \( {x}_{t} \) via the feasibility oracle \( h \). If \( {x}_{t} \) is feasible, evaluate its objective value and include it in the feasible set \( \mathcal{D}_f \). Otherwise, apply the post-decoding process \( \psi_{\mathcal{D}_f} \) before evaluating the objective value.
(5) Update the observations \( Z_t \) and \( Y_t \) accordingly.
The \texttt{CageBO} algorithm iterates a total of $T$ times.
The proposed method is summarized in Algorithm \ref{algo_BO}. 
In the remainder of this section, we explain each component of our proposed method at length.

\subsection{Conditional Generative Representation} \label{FSL}

To address the challenge of implicit constraints in ICBBO problems, we introduce a conditional generative representation model based on a CVAE in our framework. 
The rationale behind our approach is threefold: (1) To encode the original decision space with implicit constraints into a compact, continuous, and constraint-free latent space. (2) To condense the dimensionality of the original problem, making BBO more efficient in the latent space. (3) To actively search for solutions with a high likelihood of being feasible. 
As illustrated by Figure~\ref{fig:trajectory}, we observe that BBO can navigate within a feasible latent space, which offers a simpler structure in the objective function, facilitated by the conditional generative representation model.

Suppose there exists a joint distribution between the decision variable \( x \in \mathcal{X} \), which can be defined in a high-dimensional hybrid space of discrete and continuous variables, and a continuous latent variable \( z \in \mathcal{Z} \subseteq \mathbb{R}^{d} \), conditioned on the feasibility criterion \( c \in \{0, 1\} \).
We model the conditional distributions \( q_{\phi}(z | x, c) \), \( p_{\theta}(x|z, c) \), and the conditional prior $p_{}(z|c)$ through the use of neural networks as outlined in \citep{pinheiro2021variational}. Here, \( \phi \) and \( \theta \) represent the neural network weights associated with each respective distribution.
Since it is intractable to directly optimize the marginal maximum likelihood of $x$, the evidence lower bound (ELBO) \citep{jordan1999introduction} of the log-likelihood is derived as follows
\begin{align}
\mathcal{L}_{\mathrm{ELBO}} &= w(c)\underset{q_{\phi}(z | x, c)}{\mathbb{E}}\left[\log p_{\theta}(x|z, c)\right] \nonumber \\
&\quad -\eta \mathrm{D}_{\mathrm{KL}}\left(q_{\phi}(z | x, c) \| p(z|c)\right),\label{eq:elbo}
\end{align}
where \( w(c) \) denotes the weighting function associated with the feasibility \( c \), and \( \eta \) is a hyperparameter controlling the penalty ratio. Here we assume the conditional prior of $z$ follows a standard Gaussian distribution $\mathcal{N}(0, I)$. 
The first term in \eqref{eq:elbo} can be considered as the reconstruction error between the input and reconstructed decisions, and the second term is the Kullback–Leibler (KL) divergence between the conditional prior of the latent variable $z$ and the learned posterior $q_{\phi}(z | x, c)$.  

\begin{algorithm}[!t]
\caption{\texttt{CageBO}}\label{algo_BO}
\begin{algorithmic}
\STATE {\bfseries Input:} Labeled decisions $\mathcal{D}$; Objective function $\hat{f}$; Feasibility oracle $h$; Hyper-parameters $T$.

\STATE Train a CVAE$(p_{\theta}, q_{\phi})$ by maximizing $\mathcal{L}_{\mathrm{ELBO}}$ in $\mathcal{D}$;

\STATE Feasible set $\mathcal{D}_f = \{ x_i \in \mathcal{X} \mid (x_i, c_i) \in \mathcal{D} \text{ and } c_i = 1 \}$;

\STATE Initialize $X_0 \subseteq \mathcal{D}_f$, $Y_0 = \{\hat{f}({x}) \mid {x}\in X_0\}$;

$Z_0 \leftarrow \encoder_{\phi}(X_0)$; \hfill // Encoding

\FOR{$t = 1 \textbf{ to } T$}

\STATE Train $\text{GP}(\mu, k; \mathcal{Z})$ using $(Z_{t-1}, Y_{t-1})$;

\STATE $\hat{Z} = \mathrm{supp}(q_{\phi}(\cdot | c=1))$;


\STATE ${z}_t \in \argmin_{z\in \Hat{Z}} \operatorname{LCB}\left(z, \text{GP} \right)$;

\STATE ${x_{t} } \leftarrow \decoder_{\theta}({z}_t)$; \hfill // Decoding

\IF{$h(x_t) = 1$}{
    \STATE $y_{t} = \hat{f}({x}_{t})$, $\mathcal{D}_f \leftarrow \mathcal{D}_f \cup \{x_t\}$;
}
\ELSE

   \STATE  $y_t = \hat{f}(\psi_{\mathcal{D}_f}({x_t}))$; \hfill // Post-decoding

\ENDIF

\STATE $Z_t \leftarrow Z_{t-1} \cup \{z_{t}\} $, $Y_{t} \leftarrow Y_{t-1} \cup \{y_t\}$;

\ENDFOR

\textbf{return} $x^* \in \text{argmin}_{{x}\in X_T} \hat{f}({x})$
\end{algorithmic}
\end{algorithm}

\textbf{Post-decoding process.} 
To ensure the decoded decision ${x} \in \cX$ is subject to the implicit constraints, we introduce a post-decoding process in addition to the decoder, denoted by $\psi_{\mathcal{S}}: \cX \mapsto \cS$.
This function projects any given decision ${x} \in \cX$ to the closest feasible decision $\hat{x} \in \cS$. However, the presence of implicit constraints prevents us from achieving an exact projection.
As a workaround, we search within the observed feasible set $\mathcal{D}_f = \{ x_i \in \mathcal{X} \mid (x_i, c_i) \in \mathcal{D} \text{ and } c_i = 1 \}$, rather than the unattainable feasible region $\mathcal{S}$, and find a feasible decision $\hat{x} \in D_f$ that is the closest to the decoded decision ${x}$ as an approximate. The distance between any two decisions is measured by the Euclidean norm $||\cdot||_2$. Formally,
\begin{equation}
    \psi_{\mathcal{D}_f}({x}) = \argmin_{\hat{x} \in \mathcal{D}_f} ||\hat{x} - x||_{2}.
\end{equation}
The post-decoding process is initiated only for decoder-generated decisions \( x \) that are infeasible. Feasible decisions are directly incorporated into the observed feasible set \( \mathcal{D}_f \). Through iteratively expanding this observed feasible set, we enhance the accuracy of this process by improving the coverage of $\mathcal{D}_f$ on the underlying feasible region \( \mathcal{S} \). This is further detailed in the ablation study in the appendix
\footnote{Appendix is available on \url{https://arxiv.org/abs/2310.18449}.\\
Code is available on \url{https://github.com/wenqian-xing/CageBO}.}.

\subsection{Surrogate Model} \label{SM}

Now we define an \textit{indirect objective function} $g(\cdot)$ which maps from $\mathcal{Z} \mapsto \mathbb{R}$:
\begin{equation}
g(z) = 
\begin{cases} 
f(\decoder_{\theta}(z)), & \text{if } h(\decoder_{\theta}(z)) = 1, \\
f(\psi_{\mathcal{D}_f}(\decoder_{\theta}(z))), & \text{if } h(\decoder_{\theta}(z)) = 0.
\end{cases}
\end{equation}
The indirect objective function measures the objective value of the latent variable via the decoding and the potential post-decoding process. Given that the objective function \( f \) is inherently a black-box function, it follows that the indirect objective function \( g \) is also a black-box function.

We use a GP as our surrogate model of the indirect objective function $g$, denoted by $\text{GP}({\mu}, {k}; \mathcal{Z})$. 
In our problem, the mean function can be written as ${\mu(z) = \mathbb{E}[g(z)]}$. 
In addition, we adopt the Matérn kernel \citep{seeger2004gaussian} as the kernel function $k(z, z^\prime) = \mathbb{E}[(g(z) - \mu(z))(g(z^\prime) - \mu(z^\prime))]$, which is widely-used in BO literature. 
The main advantage of the GP as a surrogate model is that it can produce estimates of the mean evaluation and variance of a new latent variable, which can be used to model uncertainty and confidence levels for the acquisition function described in the following. 
Note that the latent variable $z$ is assumed to follow a Gaussian prior in the latent decision model, which aligns with the assumption of the GP model that the observed latent variables $Z$ follow the multivariate Gaussian distribution.

\begin{figure}[t!]
  \centering
  \begin{subfigure}[b]{0.5\linewidth}  
    \centering
    \includegraphics[width=\linewidth]{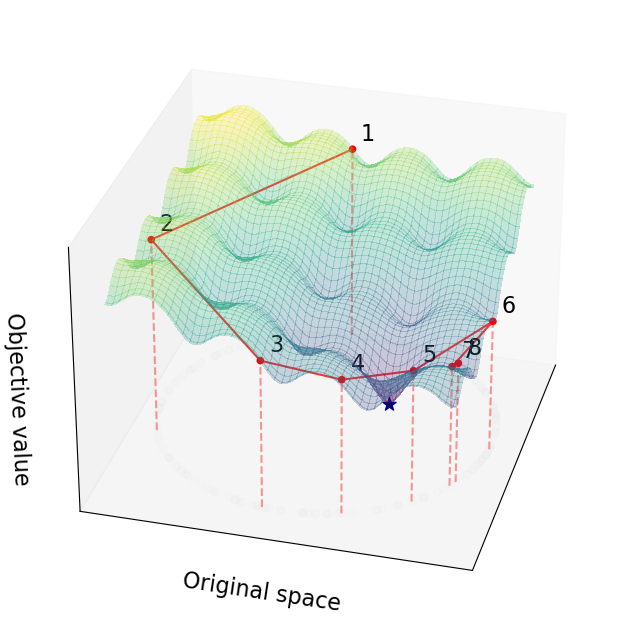}
    \caption{Original space}
  \end{subfigure}
  \hfill
  \begin{subfigure}[b]{0.48\linewidth}  
    \centering
    \includegraphics[width=\linewidth]{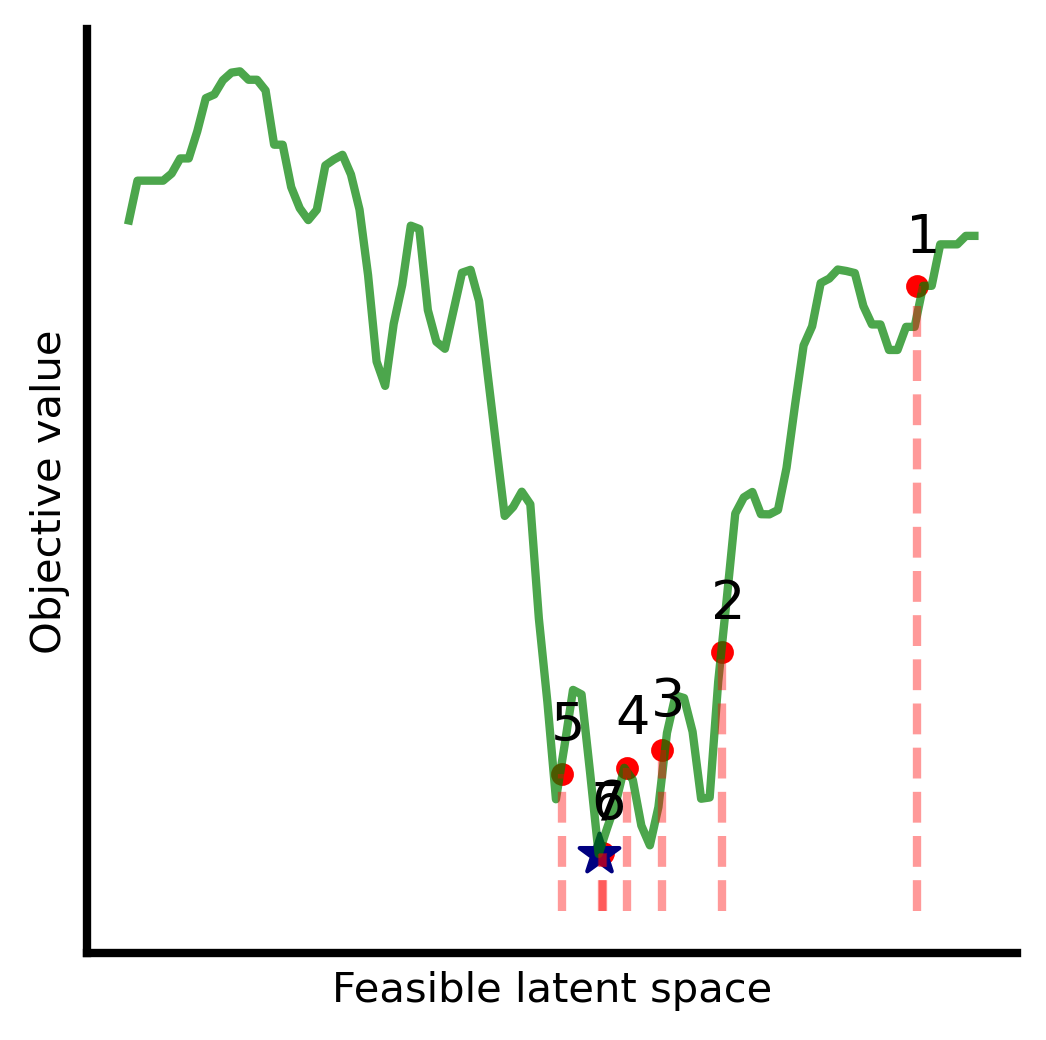}
    \caption{Feasible latent space}
  \end{subfigure}
  \caption{Solution paths of the same optimization problem suggested by our \texttt{CageBO} algorithm in the (a) original and (b) latent spaces, respectively. The objective in this illustrative example is to minimize the Ackley function subject to an "implicit" constraint (a circle) in a 2-dimensional decision space.}
  \label{fig:trajectory}
\end{figure}

\textbf{Acquisition function.} 
In the BO methods, the acquisition function is used to suggest the next evaluating candidate. 
Our approach adopts the lower confidence bound (LCB) as the acquisition function to choose the next latent variable $z$ candidate to be decoded and evaluated. 
This function contains both the mean $\mu(z)$ of the GP as the explicit exploitation term and the standard deviation $\sigma(z)$ of the GP as the exploration term:
\begin{equation}
\mathrm{LCB}(z, \text{GP}) = \mu(z) - \sqrt{\beta}\sigma(z),
\end{equation}
where $\beta$ is a trade-off parameter. 

To identify new latent decisions for evaluation, our method draws numerous independent samples from the conditional posterior \( q_{\phi}(\cdot |x, c=1) \) among observed feasible decisions $x\in \mathcal{D}_f$. The sample with the lowest LCB value, denoted as \( z_t \), is selected for decoding and subsequent evaluation using the objective function \( \hat{f} \) in the \( t \)-th iteration.

\subsection{Theoretical Analysis}\label{sec:theory}
We provide theoretical analysis for Algorithm \ref{algo_BO}. Let $\mathtt{enc}: \cX \mapsto \cZ$ denote the encoder function, $\mathtt{dec}: \cZ \mapsto \cX$ denote the decoder function, and $g(z) = f(\psi_{\mathcal{D}_f}(\mathtt{dec}(z))): \cZ \mapsto \mathbb{R}$ denotes the objective function w.r.t. $\cZ$ where $\psi_{\mathcal{D}_f}$ is the post-decoder. Note even if post-decoder is not needed, $\psi_{\mathcal{D}_f}(\mathtt{dec}(z))=\mathtt{dec}(z)$.
Following the existing work in Gaussian process bandit optimization, we utilize cumulative regret and its expected version to evaluate the performance of our algorithm which are defined as follows.
\begin{equation}
    R_T = \sum_{i=1}^T f(x_t) - f(x_*), 
\E[R_T] = \sum_{i=1}^T \E[f(x_t) - f(x_*)],
\end{equation}
where $x_*=\argmin_{x \in \cX} f(x)$ and the expectation is taken over all randomness, including random noise and random sampling over observations.

We assume the distance between any two points in $\cZ$ can be upper bounded by their distance in $\cX$, i.e., $\forall x, x' \in \cX, \|\mathtt{enc}(x) - \mathtt{enc}(x')\|_2 \leq C_p \|x - x'\|_2$. We further assume that function $g: \cZ \mapsto \mathbb{R}$ is drawn from a Gaussian Process prior and it is $C_g$-Lipschitz continuous, i.e., $\forall z, z' \in \cZ, |g(z) - g(z')| \leq C_g \|z - z'\|_2$. This assumption is the standard Gaussian process assumption for Bayesian optimization \citep{srinivas2010gaussian}. Here $C_p$ and $C_g$ are universal constants.
We use big $\tilde{O}$ notation to omit any constant and logarithmic terms.
Now we are ready to state our main theoretical result.
\begin{theorem}\label{thm:main}
After running $T$ iterations, the expected cumulative regret of Algorithm \ref{algo_BO} satisfies that
\begin{align}
\E[R_T] = \widetilde{O}(\sqrt{T \gamma_T} + \sqrt{d} (n+T)^\frac{d}{d+1}),
\end{align}
where $\gamma_T$ is the maximum information gain, depending on choice of kernel used in algorithm and $n$ is number of initial observation data points.
\end{theorem}
\begin{remark}
We present an upper bound of expected cumulative regret of our Algorithm \ref{algo_BO}. This is a no-regret algorithm since $\lim_{T \rightarrow \infty} \E[R_T]/T=0$ which means our algorithm is able to find the global optimal solution by expectation. The bound has two terms. The first term follows from GP-UCB \citep{srinivas2010gaussian} where the maximum information gain depends on the choice of kernel used in the algorithm. For linear kernel, $\gamma_T = O(d \log T)$ and for squared exponential kernel, $\gamma_T = O(\log T)^{d+1}$. The second term is the regret term incurred by the post decoder $\psi_{\mathcal{D}_f}$ when needed, which is also sublinear in $T$ if $n$ is chosen no larger than $T$. Due to page limit, full proof of Theorem \ref{thm:main} is shown in the appendix.
\end{remark}

\section{Experiments}\label{sec:exp}

\subsection{Experimental Setup}\label{sec:exp_setup}

We compare our \texttt{CageBO} algorithm against three baseline approaches that can be used to address ICBBO problems. These baselines include (1) vanilla version of Bayesian optimization (\texttt{BO}) \citep{jones1998efficient}, (2) simulated annealing (\texttt{SA}) \citep{kirkpatrick1983optimization}, (3) approximated Mixed Integer Linear Programming (\texttt{MILP}) \citep{zhu2022data}, (4) a VAE-guided Bayesian optimization (\texttt{VAE-BO}) \citep{eissman2018bayesian}, and (5) a variant of constrained VAE-guided Bayesian optimization (\texttt{CON-BO}) \citep{griffiths2020constrained}.
We train the latent decision model in our framework with the Adam optimizer \citep{king2015adam}.
Both \texttt{CageBO}, \texttt{VAE-BO}, and \texttt{CON-BO} are trained and performed under an identical environment. \texttt{MILP} is only compared in the redistricting problems because their implicit constraints can be approximated by a set of linear constraints, albeit with the trade-offs of adding auxiliary variables and incurring computational expenses.
Each method is executed 10 times across all experiments to determine the $95\%$ confidence interval of their results.

Experiments were conducted on a PC equipped with a 12-core CPU and 18 GB RAM. For synthetic experiments, \texttt{CageBO} is trained in $1,000$ epoches with the learning rate of $10^{-4}$, $\eta=0.1$, and dimension of latent space $d = 10$. Hyperparameters: $|X_0| = 10$ initial evaluation points, $\beta=1$ for the LCB, and $T=100$. 
For real-world redistricting experiments, \texttt{CageBO} is trained in $1,000$ epochs with the learning rate of $10^{-4}$, $\eta=0.1$, and the dimension of latent space $d = 25$. Hyperparameters: $|X_0| = 5$ initial evaluation points, $\beta=1$ for the LCB, and total iterations $T=100$.

\subsection{Synthetic Results}
\label{sec:synthetic}

\begin{figure}[!t]
    \centering
    \begin{subfigure}[b]{0.5\linewidth}
        \centering
        \includegraphics[width=\linewidth]{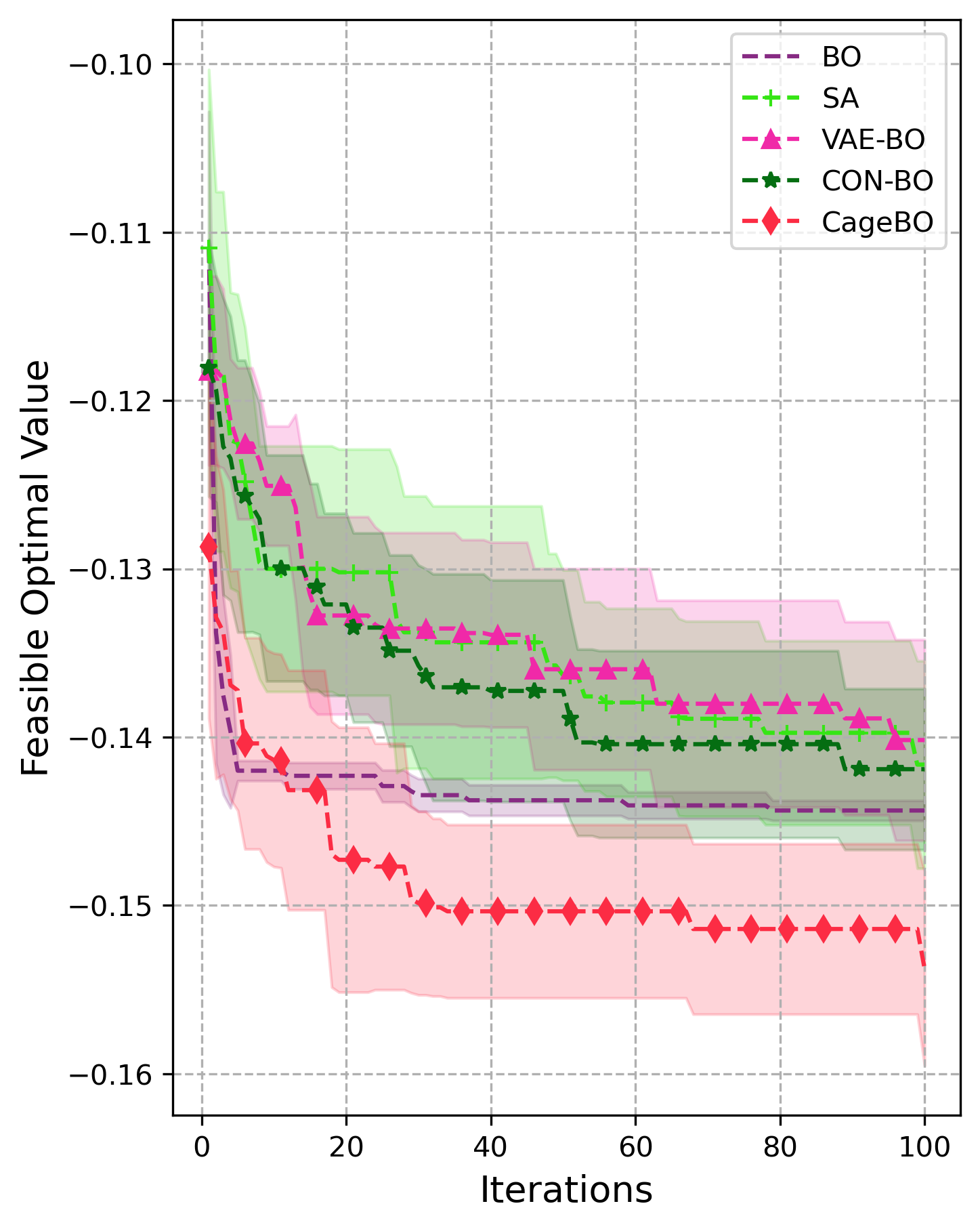}
        \caption{Keane's bump function}
        \label{fig:keane_all}
    \end{subfigure}
    \hfill
    \begin{subfigure}[b]{0.48\linewidth}
        \centering
        \includegraphics[width=\linewidth]{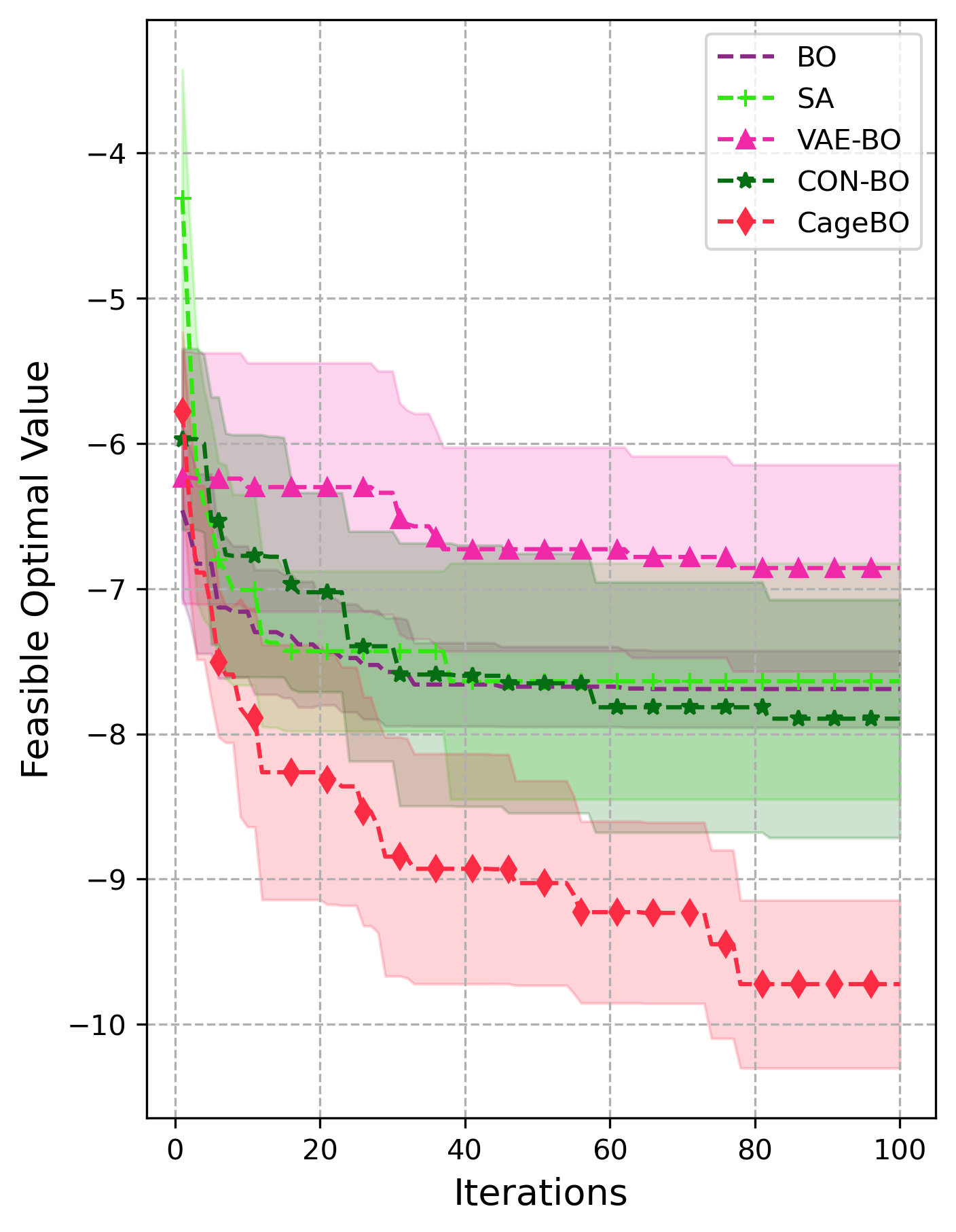}
        \caption{Michalewicz function}
        \label{fig:Michalewicz_all}
    \end{subfigure}
    \hfill
    \begin{subfigure}[b]{0.48\linewidth}
        \centering
        \includegraphics[width=\linewidth]{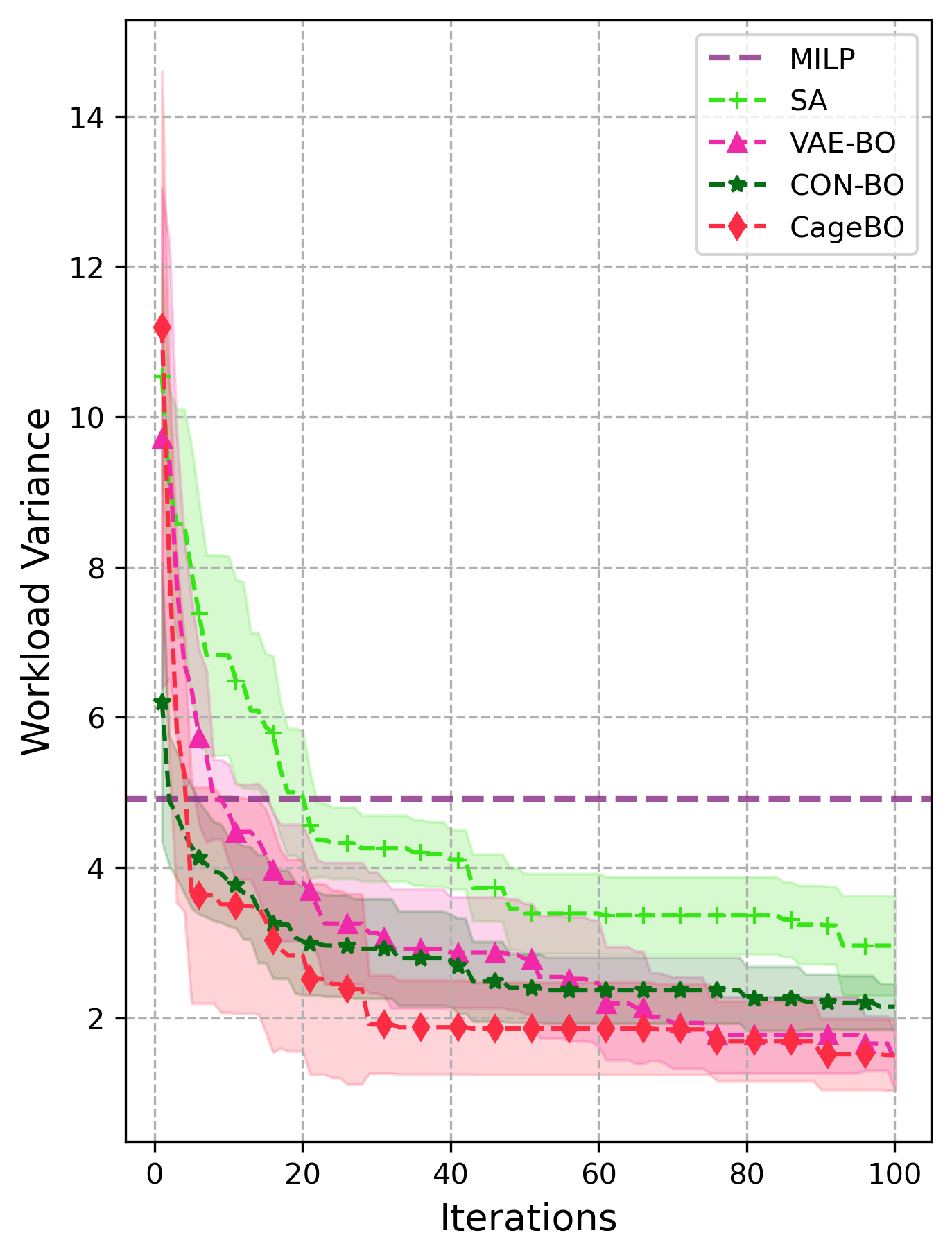}
        \caption{Synthetic $6 \times 6$ grid}
        \label{fig:synthetic_grid}
    \end{subfigure}
    \hfill
    \begin{subfigure}[b]{0.485\linewidth}
        \centering
        \includegraphics[width=\linewidth]{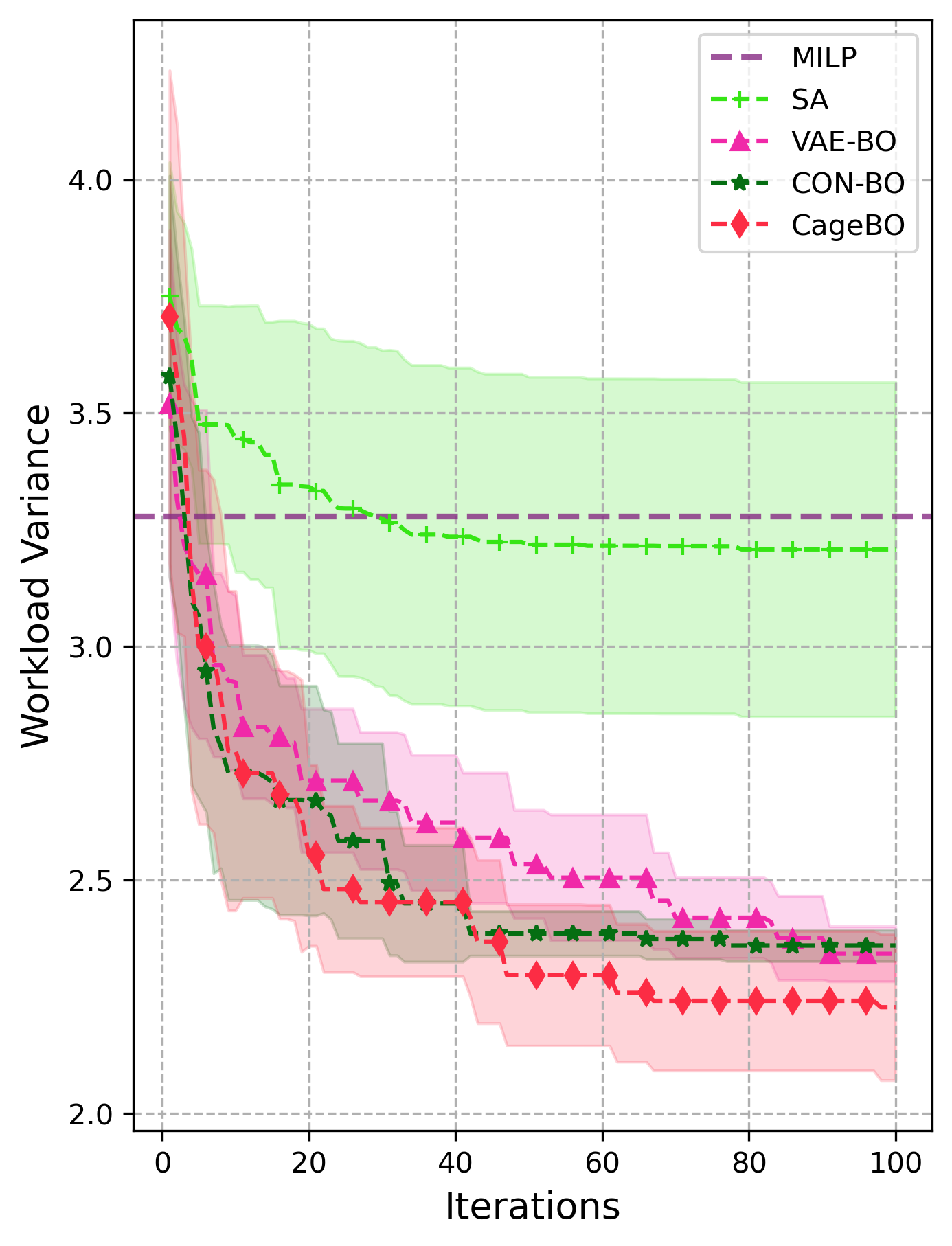}
        \caption{Atlanta}
        \label{fig:real_data}
    \end{subfigure}
    
    \caption{
    Comparisons of performance convergence with 95\% confidence intervals for (a) 30-dimensional Keane's bump function, (b) 30-dimensional Michalewicz function, (c) police redistricting problem in a synthetic $6 \times 6$ grid, and (d) police redistricting problem in Atlanta, Georgia.
    }
    \label{fig:synthetic}
\end{figure}

We consider minimizing (a) the $30$-d Keane's bump function \citep{keane1996experiences} and (b) the $30$-d Michalewicz function \citep{molga2005test}, both of which are common test functions for constrained optimization. 
To obtain samples from implicit constraints, the conundrum we aim to address via our methodology, we first generate $n=2,000$ samples from the standard uniform distribution in the $10$-dimensional latent space, then decode half of them through a randomly initialized decoder and mark those as feasible. 
We define the feasibility oracle to return 1 only if the input solution is matched with a feasible solution.
We scale the samples accordingly so that the test functions can be evaluated under their standard domains, \ie, the Keane's bump function on $[0,10]^{d}$ and the Michalewicz function on $[0, \pi]^{d}$.

Figure \ref{fig:synthetic}(a)\&(b) presents the synthetic results. It is evident that our method attains the lowest objective values consistently compared to other baseline methods. In Figure \ref{fig:synthetic} (b), in particular, we observe that the integration of the conditional generative representation model into \texttt{CageBO} greatly enhances the BO's performance. This is in stark contrast to the similar learning-based \texttt{VAE-BO} and \texttt{CON-BO}, which do not yield satisfactory outcomes.

\subsection{Case Study: Police Redistricting} 
\label{sec:casestudy}

One common application of high-dimensional ICBBO in public policymaking is redistricting. 
For example, in police redistricting problems, the goal is to distribute $L$ police service regions across $J$ distinct zones. 
Each service region is patrolled by a single police unit.
While units within the same zone can assist each other, assistance across different zones is disallowed.  
The decision variables are defined by the assignment of a region to a zone, represented by matrix $x \in \{0, 1\}^{L\times J}$.
Here, an entry $x_{lj} = 1$ indicates region $l$ is allocated to zone $j$, and $x_{lj} = 0$ otherwise. 
A primal constraint is that each region should be assigned to only one zone. This means that for every region $l$, $\sum_{j\in [J]} x_{lj} = 1$. 
The implicit constraint to consider is the contiguity constraint, which ensures that all regions within a specific zone are adjacent. We define the feasibility oracle to return 1 only if both the primal and the implicit constraint are satisfied.

The objective of the police redistricting problem is to minimize the workload variance across all the districts. The workload for each district \( j \), denoted as \( \rho_j \), is computed using the following equation: 
\begin{equation}
\rho_j = ({\tau}_j+1/\mu) \lambda_j.
\end{equation}
In this equation, \( {\tau}_j \) represents the average travel time within district \( j \), which can be estimated using the hypercube queuing model \citep{larson1974hypercube}, which can be regarded as a costly black-box function. Assuming \( \mu \) represents the uniform service rate across all districts and \( \lambda_j \) denotes the arrival rate in district \( j \). Essentially, \( \rho_j \) quantifies the cumulative working duration per unit time for police units in district \( j \). For clarity, a value of \( \rho_j = 10 \) implies that the combined working time of all police units in district \( j \) counts to 10 hours or minutes in every given hour or minute.

\begin{figure}[]
    \centering
    \begin{subfigure}[b]{0.24\linewidth}
        \centering
        \includegraphics[width=\linewidth]{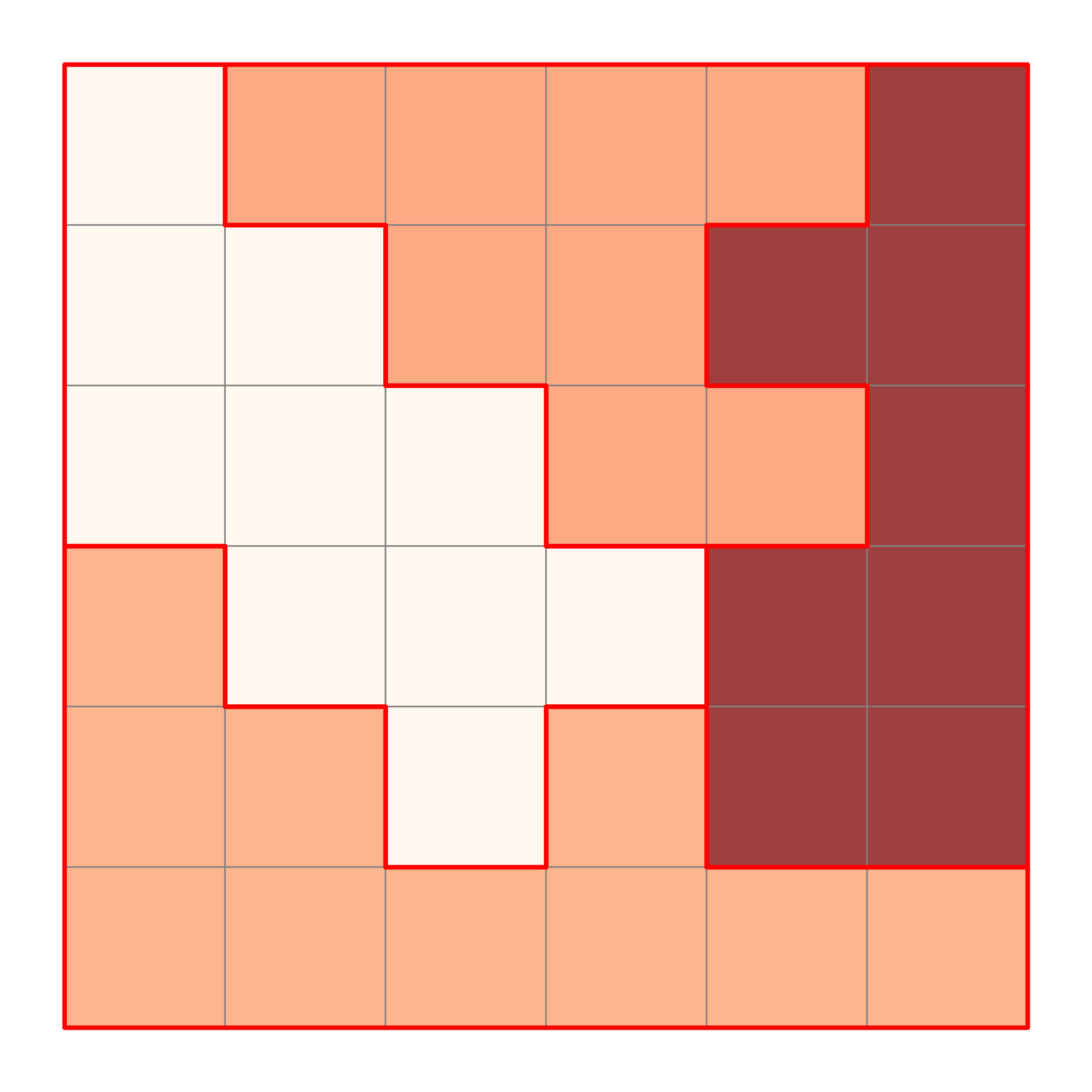}
        \caption{\texttt{MILP}}
        \label{fig:synthetic_workload_MIP}
    \end{subfigure}
    \hfill
    \begin{subfigure}[b]{0.24\linewidth}
        \centering
        \includegraphics[width=\linewidth]{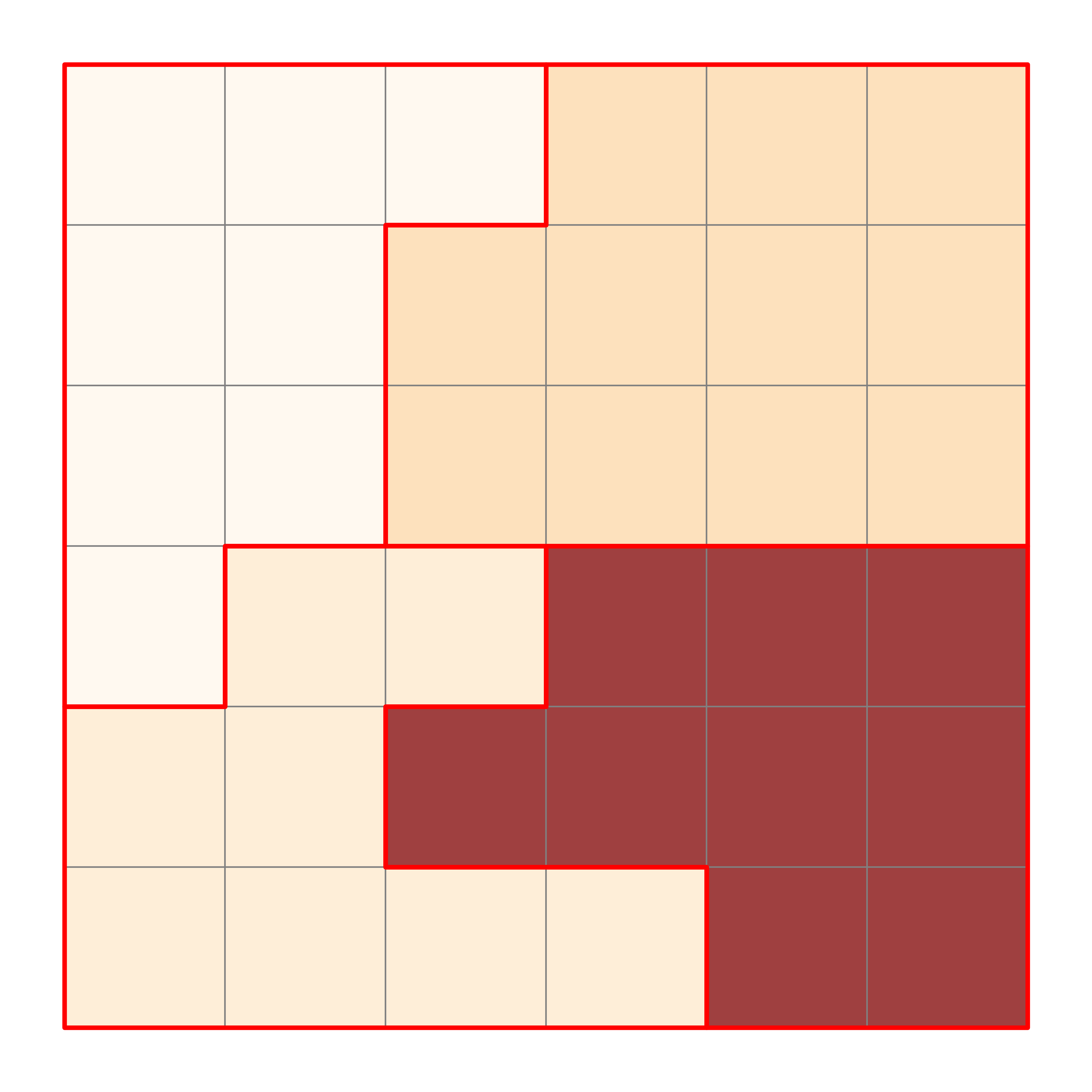}
        \caption{\texttt{SA}}
        \label{fig:synthetic_workload_SA}
    \end{subfigure}
    \hfill
    \begin{subfigure}[b]{0.24\linewidth}
        \centering
        \includegraphics[width=\linewidth]{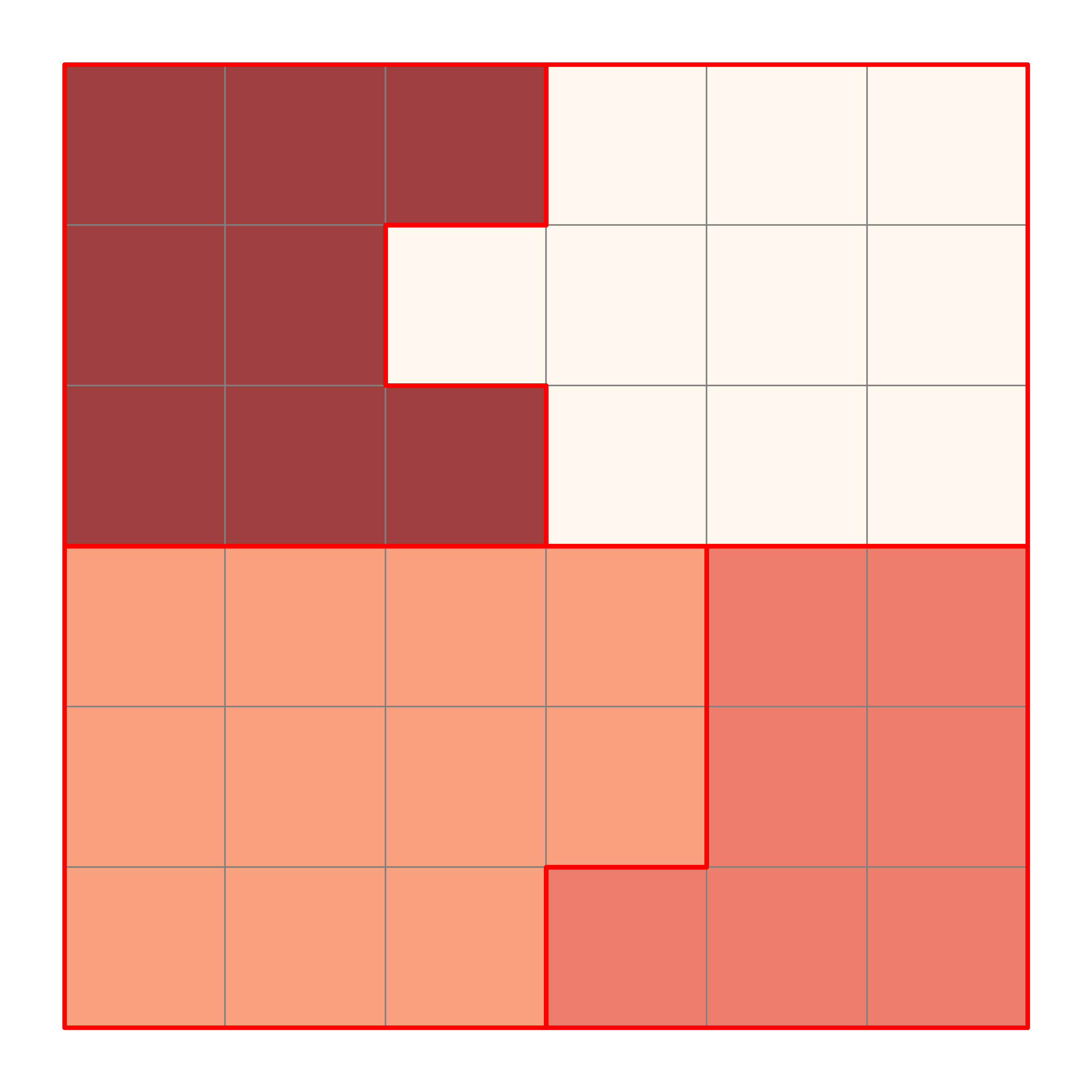}
        \caption{\texttt{VAE-BO}}
        \label{fig:synthetic_workload_GP_LCB}
    \end{subfigure}
    \hfill
    \begin{subfigure}[b]{0.24\linewidth}
        \centering
        \includegraphics[width=\linewidth]{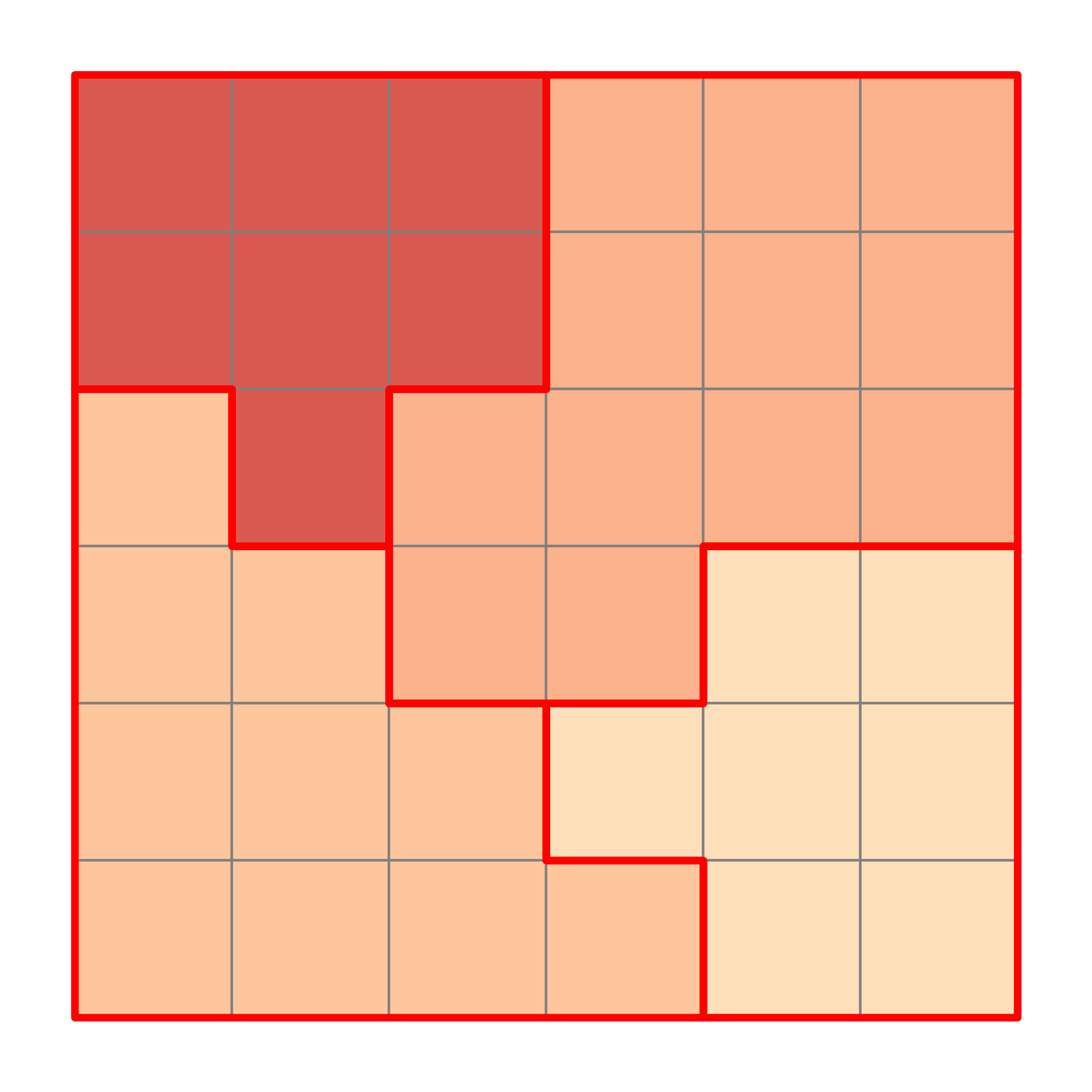}
        \caption{\texttt{CageBO}}
        \label{fig:synthetic_workload_LSBO}
    \end{subfigure}

    \vspace{0.2cm}  

    \begin{subfigure}[b]{0.24\linewidth}
        \centering
        \includegraphics[width=\linewidth]{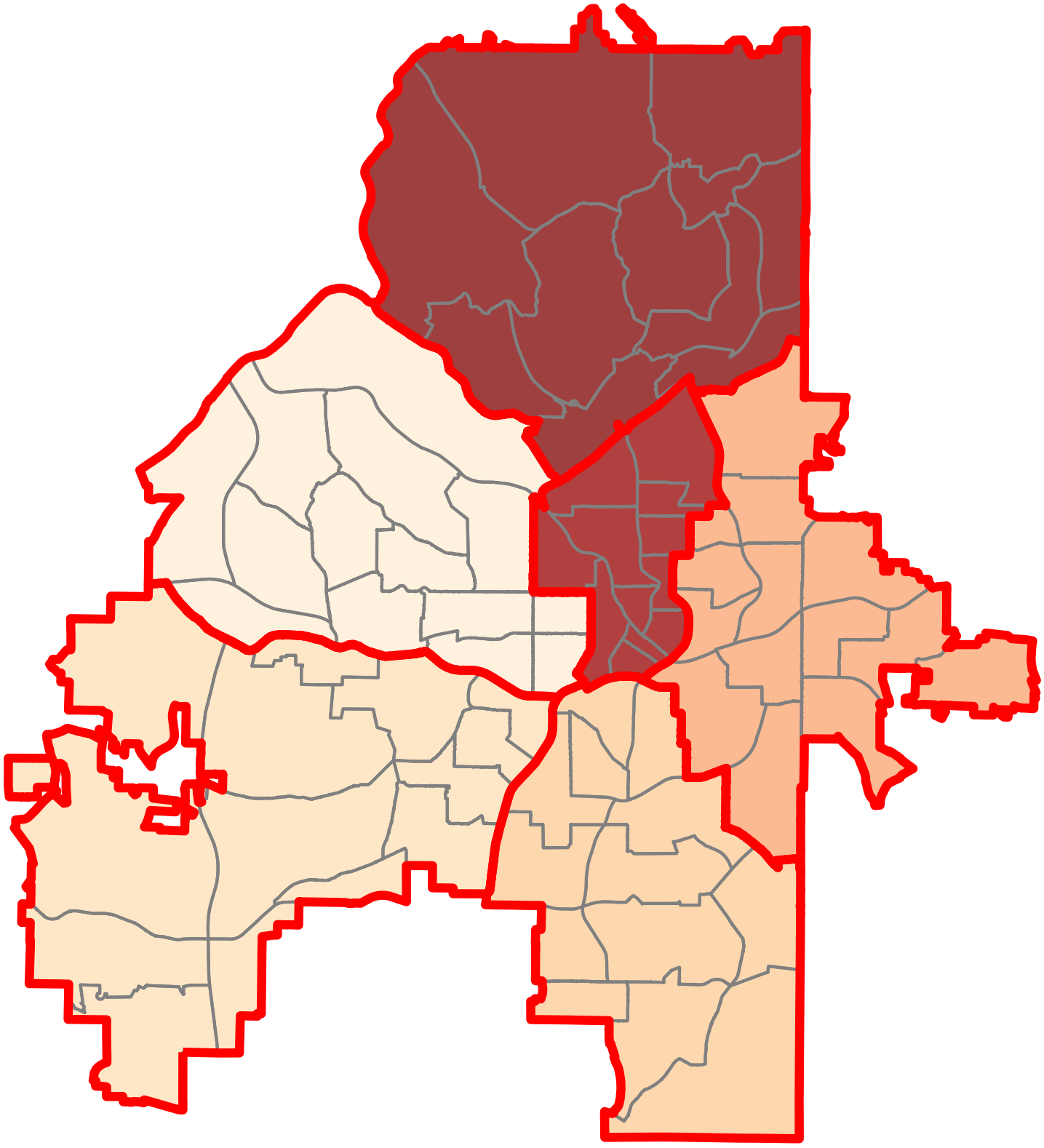}
        \caption{\texttt{MILP}}
        \label{fig:workload_0}
    \end{subfigure}
    \hfill
    \begin{subfigure}[b]{0.24\linewidth}
        \centering
        \includegraphics[width=\linewidth]{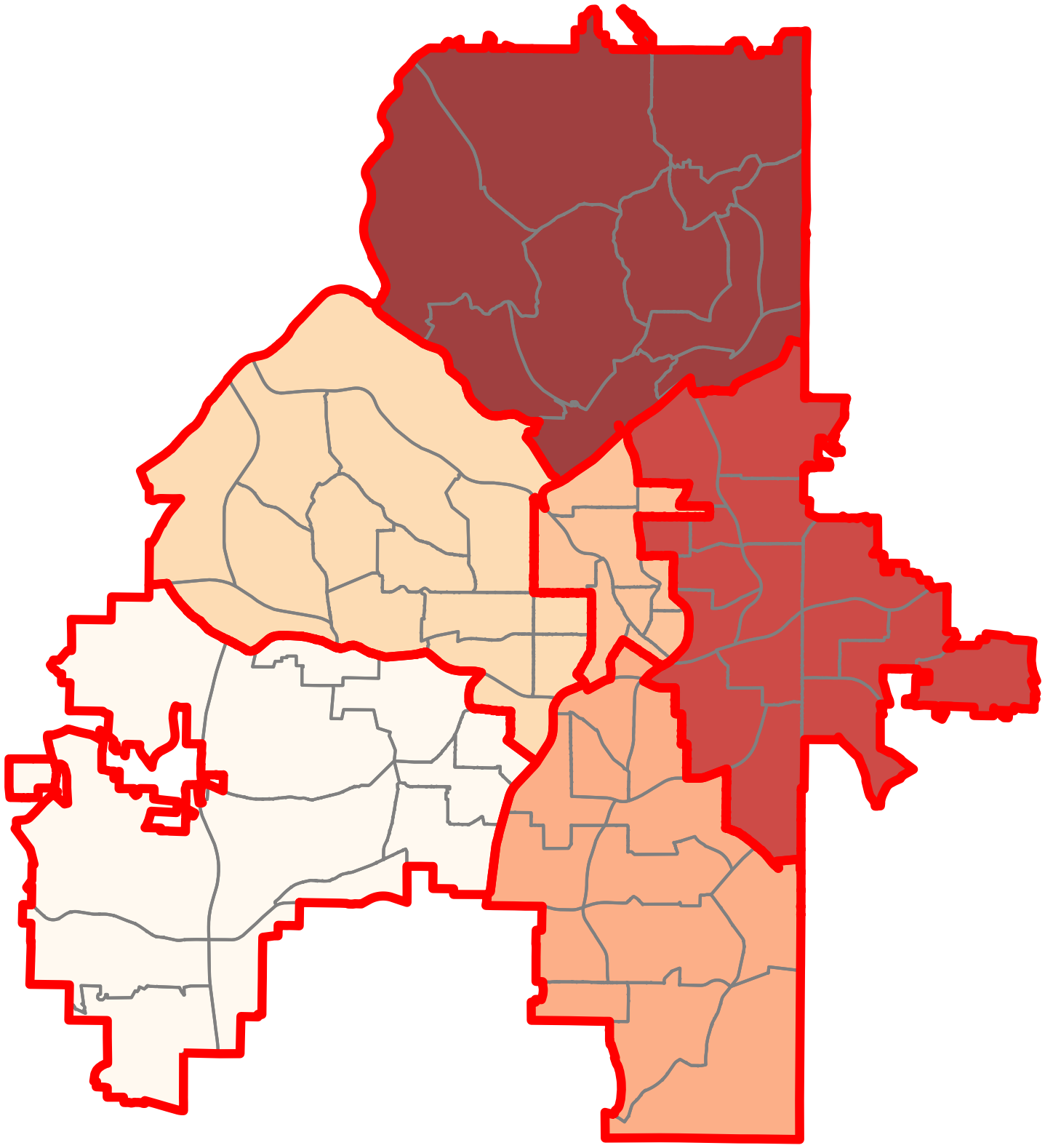}
        \caption{\texttt{SA}}
        \label{fig:workload_1}
    \end{subfigure}
    \hfill
    \begin{subfigure}[b]{0.24\linewidth}
        \centering
        \includegraphics[width=\linewidth]{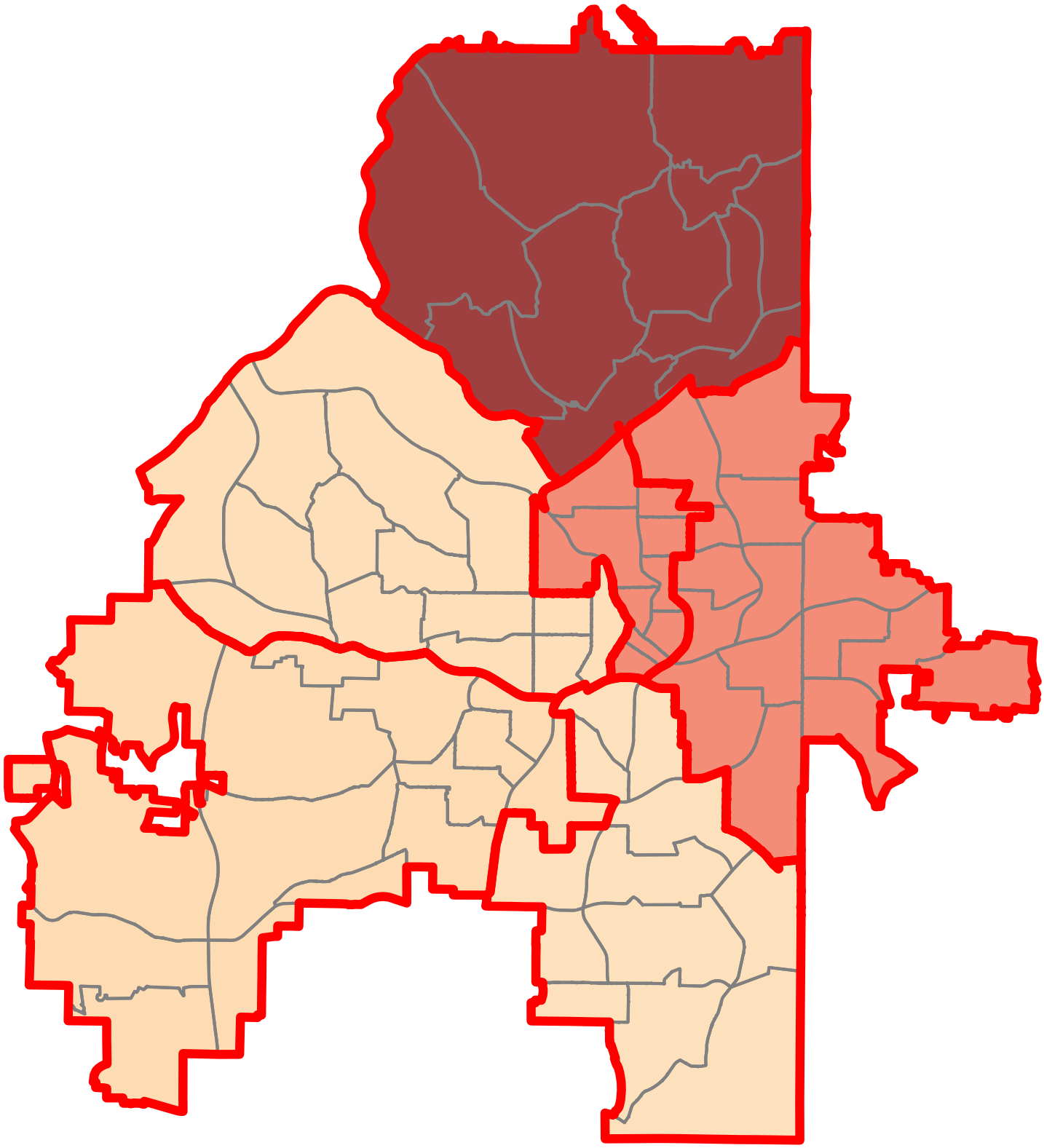}
        \caption{\texttt{VAE-BO}}
        \label{fig:workload_2}
    \end{subfigure}
    \hfill
    \begin{subfigure}[b]{0.24\linewidth}
        \centering
        \includegraphics[width=\linewidth]{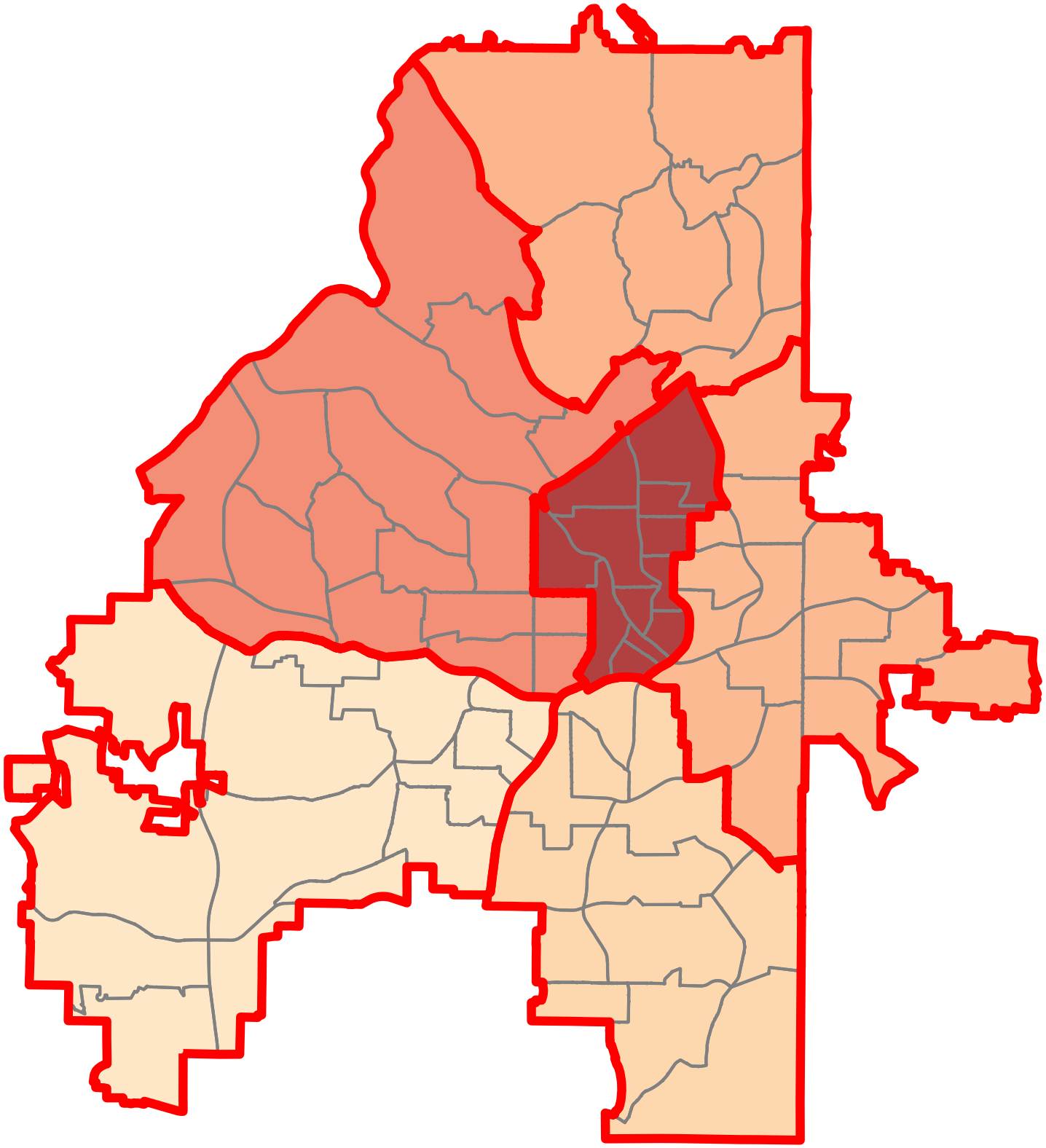}
        \caption{\texttt{CageBO}}
        \label{fig:workload_3}
    \end{subfigure}
    
    \vspace{0.2cm}  
    
    \begin{subfigure}[b]{0.8\linewidth}
    \centering
    \includegraphics[width=\linewidth]{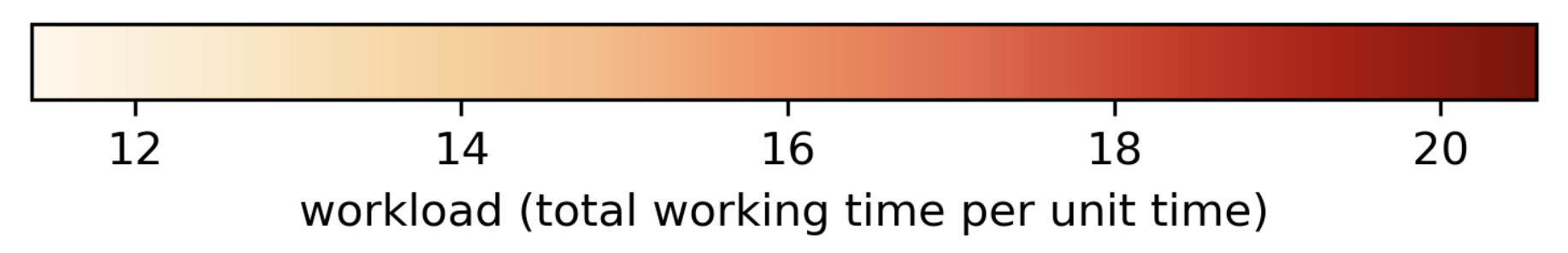}
    \label{fig:workload_colorbar}
    \end{subfigure}

    \vspace{-.4cm}
    
    \caption{
    Districting plans for the police redistricting problem in a synthetic $6 \times 6$ grid and a real-world scenario in Atlanta, Georgia. The workload in each region is indicated by the depth of the color. The \texttt{CageBO} plan achieves the lowest workload variance among all districts, indicating an optimal distribution of workload.
    }
    \label{fig:synthetic_benchmark}
\end{figure}

\textbf{Redistricting $6 \times 6$ grid.}
We first evaluate our algorithm using a synthetic scenario, consisting of a \(6 \times 6\) grid to be divided into 4 zones, and the decision variable $x \in \{0, 1\}^{36\times 4}$. For each region $l$, the arrival rate $\lambda_l$ is independently drawn from a standard uniform distribution. Moreover, all regions share an identical service rate of $\mu = 1$. We generate an initial data set $\mathcal{D}$ by simulating $n = 10,000$ labeled decisions.

Figure \ref{fig:synthetic}(c) demonstrates that our method notably surpasses other baseline methods regarding objective value and convergence speed. Additionally, in Figure \ref{fig:synthetic_benchmark}, we show the optimal districting plans derived from our algorithm alongside those from the baseline methods. It is evident that our approach produces a more balanced plan compared to the other methods.

\textbf{Atlanta police redistricting.}
For the police zone design in Atlanta, there are 78 police service regions that need to be divided into 6 zones, with the decision variable represented as $x \in \{0, 1\}^{78\times 6}$. 
The redistricting of Atlanta's police zones is constrained by several potential implicit factors. These include contiguity and compactness, as well as other practical constraints that cannot be explicitly defined.
One such implicit constraint is the need for changes to the existing plan should be taken in certain local areas. A drastic design change is undesirable because: (1) A large-scale operational change will result in high implementation costs. (2) A radical design change will usually face significant uncertainties and unpredictable risks in future operations. 
The arrival rate $\lambda_l$ and the service rate $\mu$ are estimated using historical 911-calls-for-service data collected in the years 2021 and 2022 \citep{zhu2019crime, zhu2022spatiotemporal, zhu2022data}. 
We generate an initial data set $\mathcal{D}$ by simulating $n = 10,000$ labeled decisions which are the neighbors of the existing plan with changes in certain local areas.

Figure \ref{fig:synthetic}(d) displays the convergence of our algorithm in comparison to other baseline methods, showing a consistent reduction in workload variance. 
{Notably, the redistricting plan created by our \texttt{CageBO} algorithm most closely aligns with the actual heuristic policy applied by policymakers, wherein the zone with the highest arrival rate is assigned the largest workload, and the workloads of other zones are well-balanced. Further discussion on fair districting can be found in the appendix.}
As shown in Figure \ref{fig:synthetic_benchmark}, the plan produced by our \texttt{CageBO} algorithm achieves the lowest workload variance, surpassing the baseline algorithms. The superior performance of our \texttt{CageBO} algorithm highlights its capacity to offer critical managerial insights to policymakers, thereby facilitating informed decision-making in real-world applications.

\section{Conclusion}

In this paper, we introduce a new category of optimization problems in public policy, termed Implicit-Constrained Black-Box Optimization (ICBBO), where constraints are implicit but can be easily verified. This problem is well motivated by real-world police redistricting problems where districting plans are constrained to socio-economic or political considerations but feasible plans can be easily verified by policymakers. This new ICBBO problem poses more challenges than conventional black-box optimization where existing work cannot be readily applied. 

To address this problem, we develop the \texttt{CageBO} algorithm and its core idea is to learn a conditional generative representation of feasible decisions to effectively overcome the complications posed by implicit constraints. In this way, our method can not only learn a good representation of the original space but also generate feasible samples that can be used later. 
Our theoretical analysis confirms that \texttt{CageBO} can, on expectation, identify the global optimal solution. 
This approach was empirically validated through a case study on police redistricting in Atlanta, Georgia, where our algorithm demonstrated superior performance in optimizing districting plans. 

\section*{Acknowledgments}
This work was partially done when CL was at the University of Chicago. This work was partially supported by the UAlbany Computer Science Department startup funding. The authors would like to thank the AAAI-AISI reviewers and area chair for helpful comments.

\bibliography{reference}

 \newpage
 \appendix

 \section{Appendix}
 \subsection{Details of Theoretical Analysis}\label{app:theory}

 In this section, we restate our main theorem and show its complete proof afterward.
 In this section, we restate our main theorem and provide its complete proof. We begin by formally stating the two key assumptions that will be utilized. The first assumption pertains to the properties of the encoder.
 \begin{assumption}\label{asp:coder}
 The distance between any two points in $\cZ$ can be upper bounded by their distance in $\cX$, i.e.,
 \begin{align}
 \|\mathtt{enc}(x) - \mathtt{enc}(x')\|_2 \leq C_p \|x - x'\|_2, \quad \forall x, x' \in \cX.
 \end{align}
 where $C_p$ is a universal constant.
 \end{assumption}
 The next assumption is the standard Gaussian process assumption for Bayesian optimization \citep{srinivas2010gaussian}.
 \begin{assumption}\label{asp:gp}
 Function $g: \cZ \mapsto \mathbb{R}$ is drawn from some Gaussian Process and it is $C_g$-Lipschitz continuous, i.e.,
 \begin{align}
 |g(z) - g(z')| \leq C_g \|z - z'\|_2, \quad \forall z, z' \in \cZ,
 \end{align}
 where $C_g$ is a universal constant.
 \end{assumption}

 \begin{theorem*}[Restatement of Theorem \ref{thm:main}]
 After running $T$ iterations, the expected cumulative regret of Algorithm \ref{algo_BO} satisfies that
 \begin{align}
 \E[R_T] = \widetilde{O}(\sqrt{T \gamma_T} + \sqrt{d} (n+T)^\frac{d}{d+1}),
 \end{align}
 where $\gamma_T$ is the maximum information gain, depending on the choice of kernel used in the algorithm and $n$ is the number of initial observation data points.
 \end{theorem*}

 \begin{proof}
 Our proof starts from the bounding the error term incurred in the post-decoder. Let $S$ denote the set of $n$ data points sampled {\it i.i.d.} from domain $\cX$ and $\epsilon$ denote the expected distance between any data point $x$ and its nearest neighbor $x'$ in $\cX$, \ie,
 \begin{align}
 \epsilon = \E_{x, \cX}[\|x - x'\|_2].
 \end{align}

 Recall that $\cX \subseteq [0, 1]^d$ and we discretize it in each dimension using distance $\varepsilon$ and we get $r=(1/\varepsilon)^d$ small boxes. Each small box $C_i \forall i \in 1,...,r$ is a covering set of the domain. For any two data points in the same box, we have $\|x_1 - x_2\|\leq \sqrt{d}\varepsilon$, otherwise, $\|x_1-x_2\|\leq \sqrt{d}$. Therefore,
 \begin{align}
 \epsilon \leq \E_{S} \left[\P\left[ \bigcup_{i:C_i \cap S = \emptyset} C_i \right]\sqrt{d} + \P \left[\bigcup_{i:C_i \cap S \neq \emptyset} C_i \right] \varepsilon \sqrt{d} \right].
 \end{align}

 By Lemma \ref{lem:1nn} and $\P[\cup_{i:C_i \cap S \neq \emptyset} C_i] \leq 1$ we have
 \begin{align}
 \epsilon &\leq \sqrt{d}\left(\frac{r}{ne} + \varepsilon \right)\\
 &= \sqrt{d}\left(\frac{(1/\varepsilon)^d}{ne} + \varepsilon \right)\\
 &\leq 2\sqrt{d} n^{-\frac{1}{d+1}},\label{eq:epsilon}
 \end{align}
 where the last step is by choosing $\varepsilon = n^{-1/(d+1)}$.

 Next, we try to upper bound the expected cumulative regret. Let $C_f$ denote the constant probability that a data point $x_t$ is feasible, i.e., $\P[h(x_t)=1]=C_f$, which means with probability $1-C_f$, data point $x_t$ needs to be post-decoded. By definition of cumulative regret,
 \begin{align}
 R_T & = \sum_{i=1}^T f(x_t) - f(x_*) = \sum_{i=1}^T g(z_t) - g(z_*).
 \end{align}

 Here are two events: with probability $C_f$, $\hat{z}_t$ suggested by GP-LCB is feasible and post-decoder is not needed, and with probability $1-C_f$, $\hat{z}_t$ needs to be post-decoded. Thus,
 \begin{align}
 \E[R_T] &= C_f \E\left[\sum_{i=1}^T g(\hat{z}_t) - g(z_*)\right]\\
 &+ (1-C_f) \E\left[\sum_{i=1}^T g(\hat{z}_t) - g(z_*) + g(z_t) - g(\hat{z}_t) \right],\\
 &= \E\left[\sum_{i=1}^T g(\hat{z}_t) - g(z_*)\right] + (1-C_f) \E\left[\sum_{i=1}^T g(z_t) - g(\hat{z}_t) \right].
 \end{align}

 Use Lemma \ref{lem:gpucb} and our assumptions and we have
 \begin{align}
 \E[R_T] &\leq \widetilde{O}(\sqrt{T \gamma_T}) + O\left(\sum_{t=1}^T \epsilon (1-C_f) C_p C_g \right)\\
 &\leq \widetilde{O}(\sqrt{T \gamma_T}) + O\left(\sum_{t=1}^T 2 \sqrt{d} (1-C_f) C_p C_g (n+t)^{-\frac{1}{d+1}}\right)\\
 &\leq \widetilde{O}(\sqrt{T \gamma_T} + \sqrt{d} (n+T)^\frac{d}{d+1}),
 \end{align}
 where the second inequality is due to eq. \eqref{eq:epsilon} and number of data points is $n+t$ since the algorithm keeps adding points to $\cZ$.
 \end{proof}

 \noindent\textbf{Technical lemmas.} Our proof relies on the following two lemmas:

 \begin{lemma}[Regret bound of GP-LCB (Theorem 1 of \citet{srinivas2010gaussian})]\label{lem:gpucb}
 Let $\delta \in (0,1)$ and $\beta_t=2\log(m t^2\pi^2/6\delta)$. Running GP-UCB with $\beta_t$ for a sample $f$ of a GP with mean function zero and covariance function $k(x,x')$, we obtain a regret bound of $\widetilde{O}(\sqrt{T \gamma_T \log n})$ with high probability. Precisely,
 \begin{align}
 \P[R_T \geq \sqrt{C_1 T \beta_T \gamma_T} \ \forall T \geq 1] \geq 1 - \delta,
 \end{align}
 where $C_1 = 8/\log(1+\delta^{-2})$.
 \end{lemma}
 \begin{lemma}[Expected minimum distance (Lemma 19.2 of \citet{shalev2014understanding})]\label{lem:1nn}
 Let $C_1,..., C_r$ be a collection of covering sets of some domain set $\cX$. Let $S$ be a sequence of $n$ points sampled i.i.d. according to some probability distribution $\cD$ over $\cX$. Then,
 \begin{align}
 \E_{S \sim \cD^n} \left[ \sum_{i \in C_i \cap S = \emptyset } \P(C_i) \right] \leq \frac{r}{ne}.
 \end{align}
 \end{lemma}

 \subsection{Implementation Details}
 \label{append:vae}

 Here we present the derivation of the the evidence lower bound of the log-likelihood in Eq (\ref{eq:elbo}).

 Given observation $x$, feasibility $c$, and the latent random variable $z$, let $p(x|c)$ denote the likelihood of $x$ conditioned on feasibility $c$, and $p_{}(x|z, c)$ denote the conditional distribution of $x$ given latent variable $z$ and its feasibility $c$. Let $p(z|c)$ denote the conditional prior of the latent random variable $z$ given its feasibility $c$, and $q(z|x, c)$ denote the posterior distribution of $z$ after observing $x$ and its feasibility $c$. The likelihood of observation $x$ given $c$ can be written as:
 \begin{align}
 p(x|c) &= \int p_{}(x,z|c)dz = \underset{q_{}(z | x, c)}{\mathbb{E}}\left[\frac{p_{}(x, z|c)}{q_{}(z|x, c)}\right].
 \end{align}

 By taking the logarithm on both sides and then applying Jensen’s inequality, we can get the lower bound of the log-likelihood $\mathcal{L}_{\mathrm{ELBO}}$ as follows:
 \begin{align}
 \log p(x|c) &= \log \underset{q_{}(z | x, c)}{\mathbb{E}}\left[\frac{p_{}(x, z|c)}{q_{}(z|x,c)}\right]  \\
  &\geq \underset{q_{}(z | x,c)}{\mathbb{E}}\left[\log \frac{p_{}(x, z|c)}{q_{}(z|x,c)}\right] \\
   &= \underset{q_{}(z | x,c)}{\mathbb{E}}\left[\log \frac{p_{}(x| z,c)p_{}(z|c)}{q_{}(z|x,c)}\right]   \\
     &= \underset{q_{}(z | x,c)}{\mathbb{E}}\left[\log {p_{}(x| z,c)}\right] + \underset{q_{}(z | x,c)}{\mathbb{E}}\left[\log \frac{p_{}(z|c)}{q_{}(z|x,c)}\right]     \\ 
     &= \underset{q_{}(z | x,c)}{\mathbb{E}}\left[\log p_{}(x|z,c)\right] - \mathrm{D}_{\mathrm{KL}}\left(q_{}(z | x,c) \| p_{}(z|c)\right).
 \end{align}

 In practice, we add a weight function $w(c)$ on the first term based on feasibility $c$ and a hyperparameter $\eta$ on the second term to modulate the penalization
 ratio.
 \begin{align}
 \mathcal{L}_{\mathrm{ELBO}} &= w(c)\underset{q_{}(z | x,c)}{\mathbb{E}}\left[\log p_{}(x|z,c)\right] \notag \\
 &\quad -\eta\mathrm{D}_{\mathrm{KL}}\left(q_{}(z | x,c) \| p(z|c)\right).
 \end{align}

 In practice, we approximate the conditional distribution $q(z|x, c)$ and $p(x|z, c)$ by building the $\encoder_{\phi}$ and $\decoder_{\theta}$ neural networks, and denote the approximated conditional distribution to be $q_{\phi}(z|x, c)$ and $p_{\theta}(x|z, c)$ respectively. In our formulation, the prior of the latent variable is influenced by the feasibility \( c \); however, this constraint can be easily relaxed to make the latent variables statistically independent of their labels, adopting a standard Gaussian distribution $\mathcal{N}(0, I)$. Both $q_{}(z|x,c)$ and $p_{}(x|z,c)$ are typically modeled as Gaussian distributions to facilitate closed-form KL divergence computation. Specifically, we introduce generator layers $l(\epsilon, x, c)$ and $l(\epsilon)$ to represent $q(z|x,c)$ and $p(z)$, respectively, transforming the random variable $\epsilon \sim \mathcal{N}(0, I)$ using the reparametrization trick \citep{sohl2015deep}. The log-likelihood of the first term is realized as the reconstruction loss, computed with the training data.

 \subsection{Ablation Study}
 \label{append:ablation}

 We conduct the ablation study of our \texttt{CageBO} algorithm with respect to the increasing observation sizes $n$ on both synthetic experiments and the real police redistricting problems in the $6\times 6$ grid and Atlanta, Georgia. In this study, we want to emphasize that one of the key features of our proposed conditional generative model is its ability to discover new feasible solutions. Thus, as the algorithm progresses, our CVAE model has the potential to expand the initial set of labeled observations. Figure \ref{fig:ablation_1} demonstrates that both the CVAE's performance and the subsequent decoding process improve as the observed dataset grows in size.

 \begin{figure}[H]  
     \centering
     \begin{subfigure}[b]{0.49\linewidth}
         \centering
         \includegraphics[width=\linewidth]{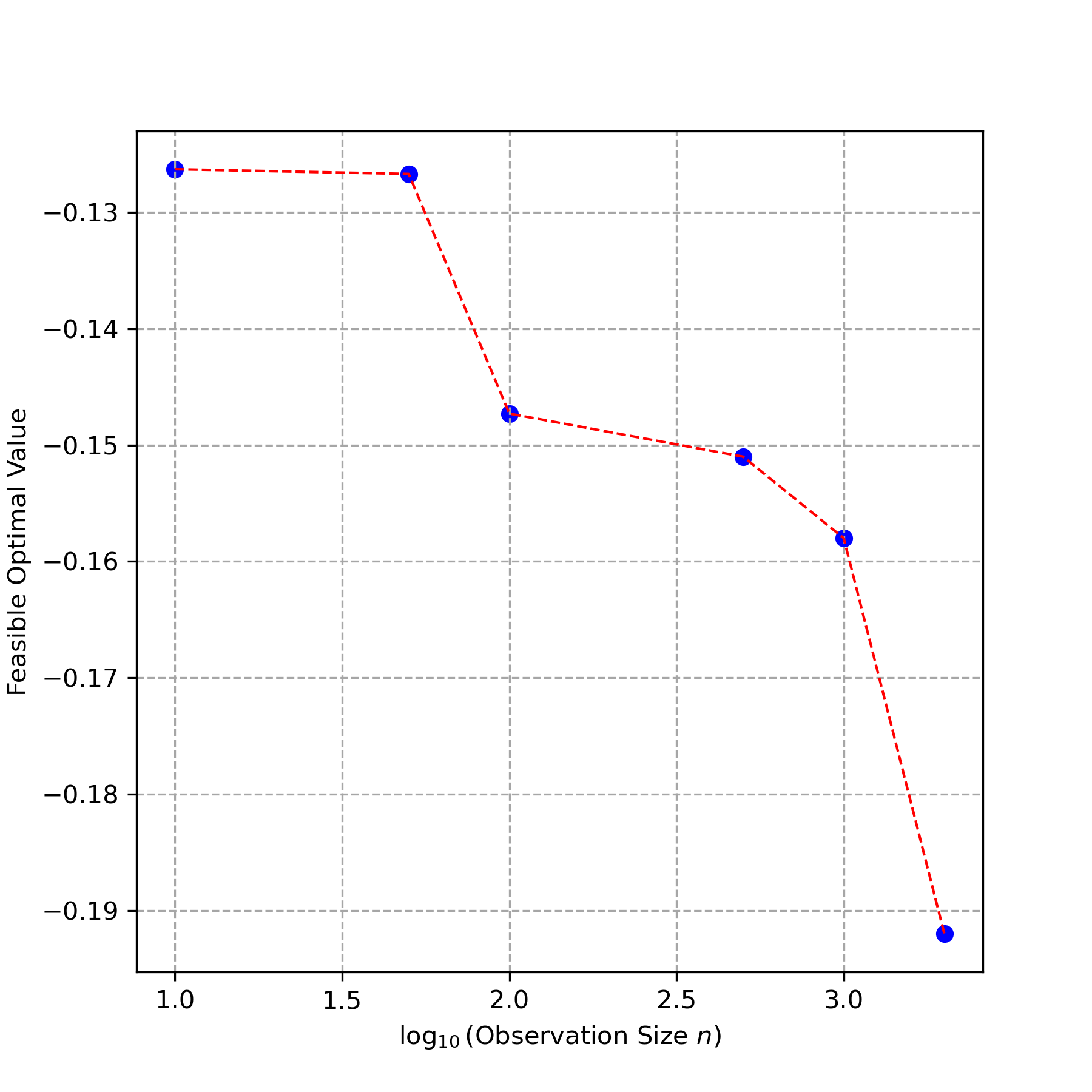}
         \caption{Keane's bump function}
     \end{subfigure}
     \begin{subfigure}[b]{0.49\linewidth}
         \centering
         \includegraphics[width=\linewidth]{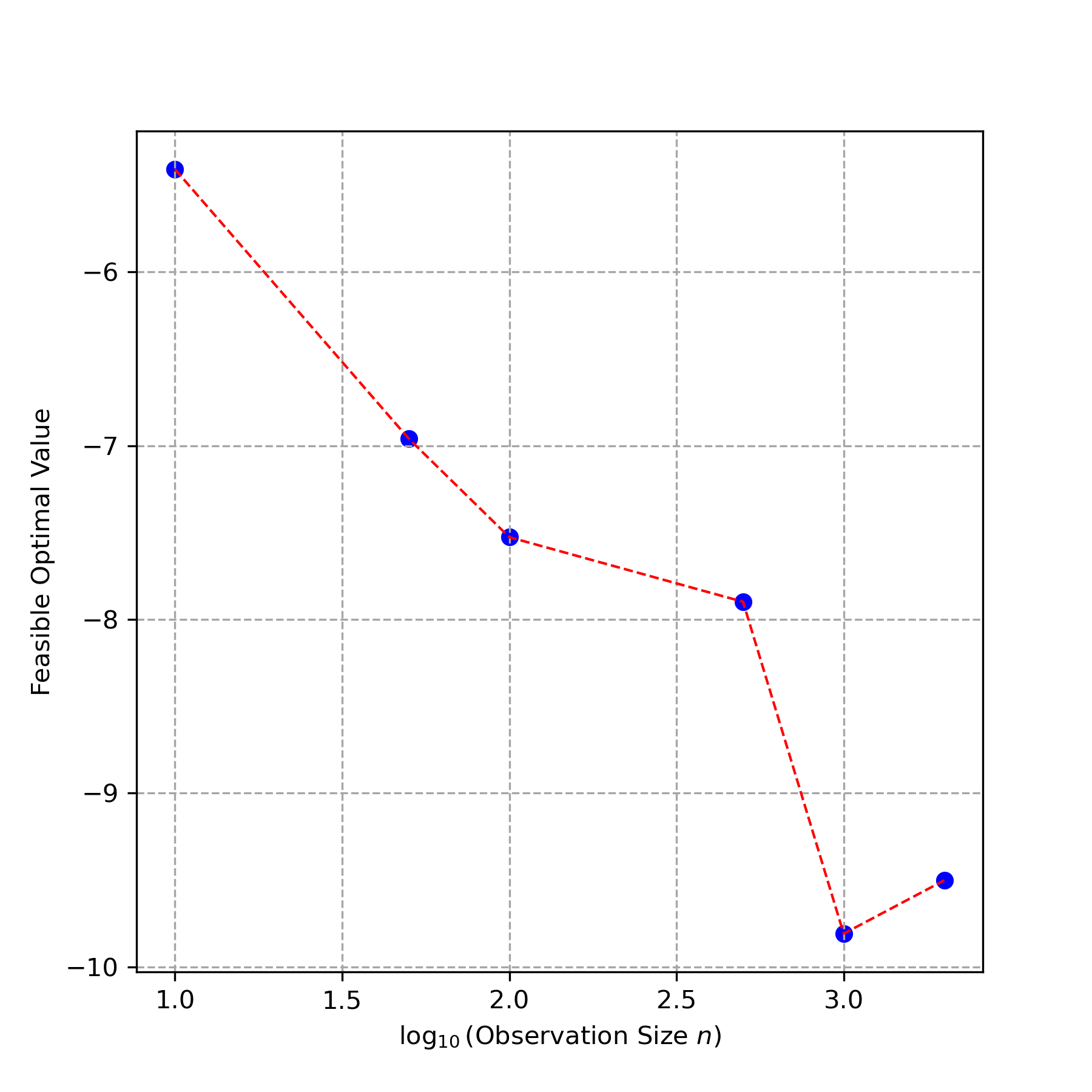}
         \caption{Michalewicz function}
     \end{subfigure}
    
    
     \begin{subfigure}[b]{0.49\linewidth}
         \centering
         \includegraphics[width=\linewidth]{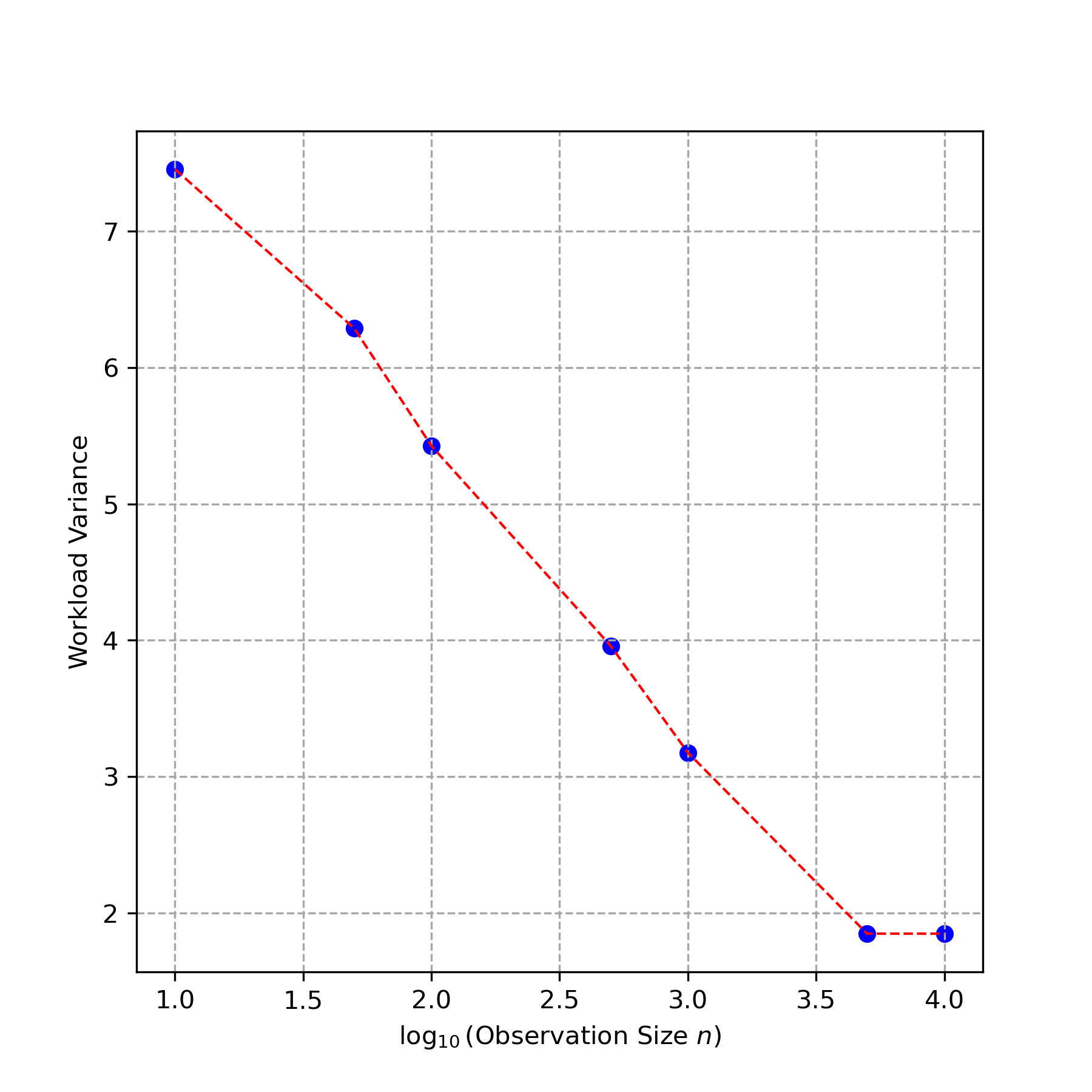}
         \caption{Synthetic $6 \times 6$ grid}
     \end{subfigure}
     \begin{subfigure}[b]{0.49\linewidth}
         \centering
         \includegraphics[width=\linewidth]{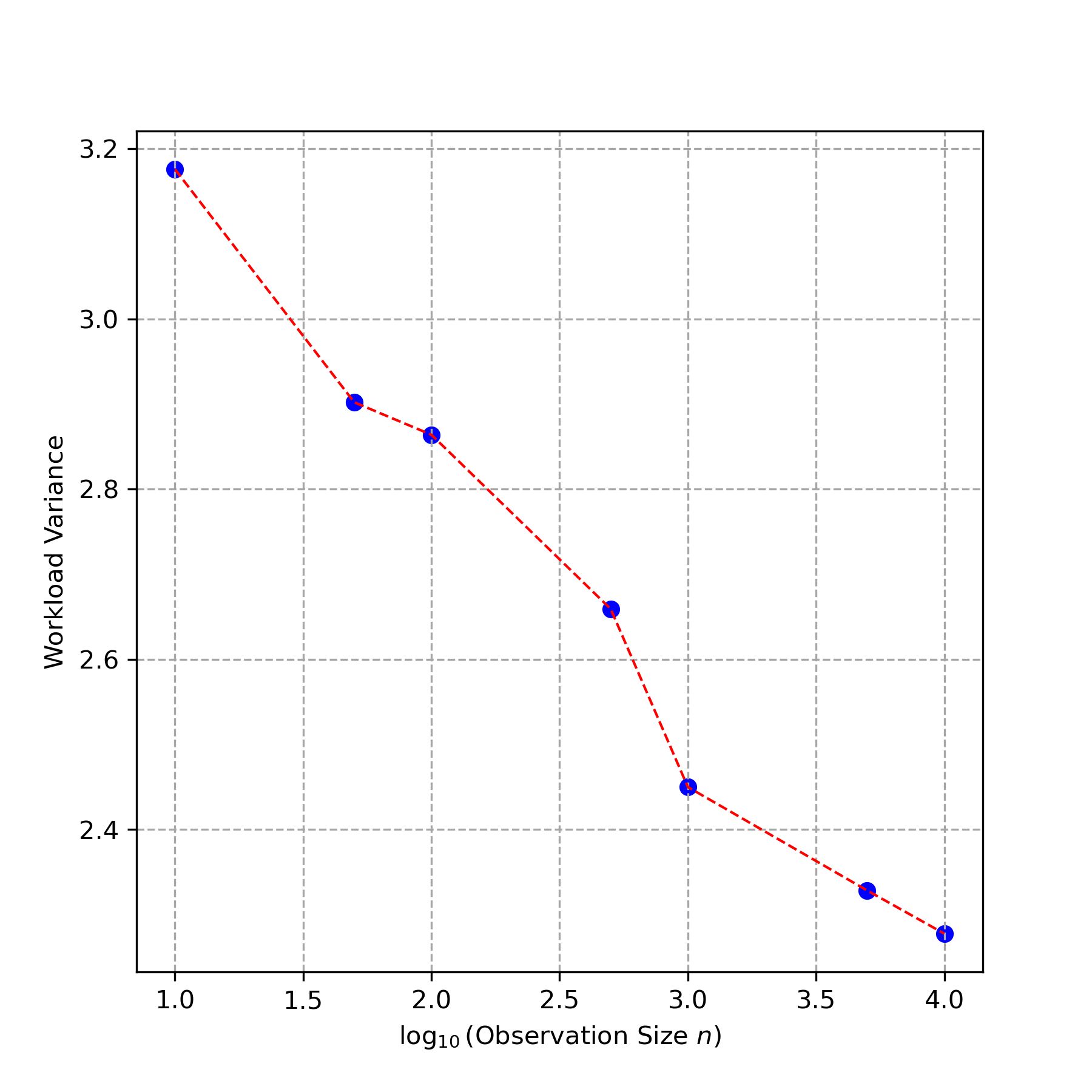}
         \caption{Atlanta}
     \end{subfigure}
    
     \caption{Optimal objective values obtained by \texttt{CageBO} with increasing observation sizes for (a) 30-dimensional Keane's bump function, (b) 30-dimensional Michalewicz function, (c) police redistricting problem in a synthetic $6 \times 6$ grid, and (d) police redistricting problem in Atlanta, Georgia. The x-axis represents the logarithm (base 10) of the observation size $n$.}
     \label{fig:ablation_1}
 \end{figure}

 Figure \ref{fig:ablation_2} (a) demonstrates that the subsequent decoding process improves as the observed dataset grows in size. Figure \ref{fig:ablation_2} (b) shows that the ability of our CAVE model to generate new feasible solutions also improves with the increasing size of the observed dataset.

 \begin{figure}[!t]
     \centering
     \begin{subfigure}[b]{0.49\linewidth}
         \centering
         \includegraphics[width=\linewidth]{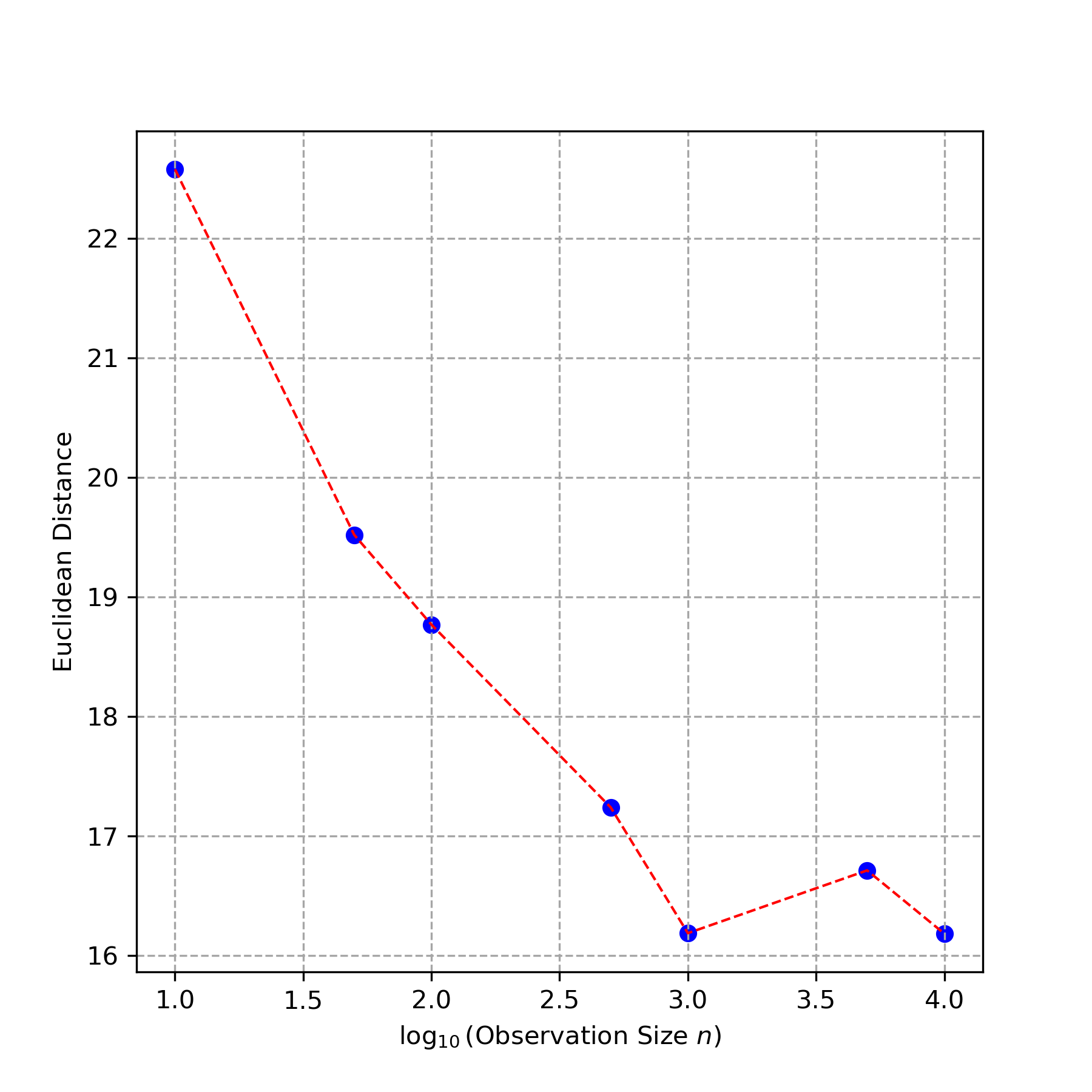}
         \caption{Distance to feasible solutions}
     \end{subfigure}
     \begin{subfigure}[b]{0.49\linewidth}
         \centering
         \includegraphics[width=\linewidth]{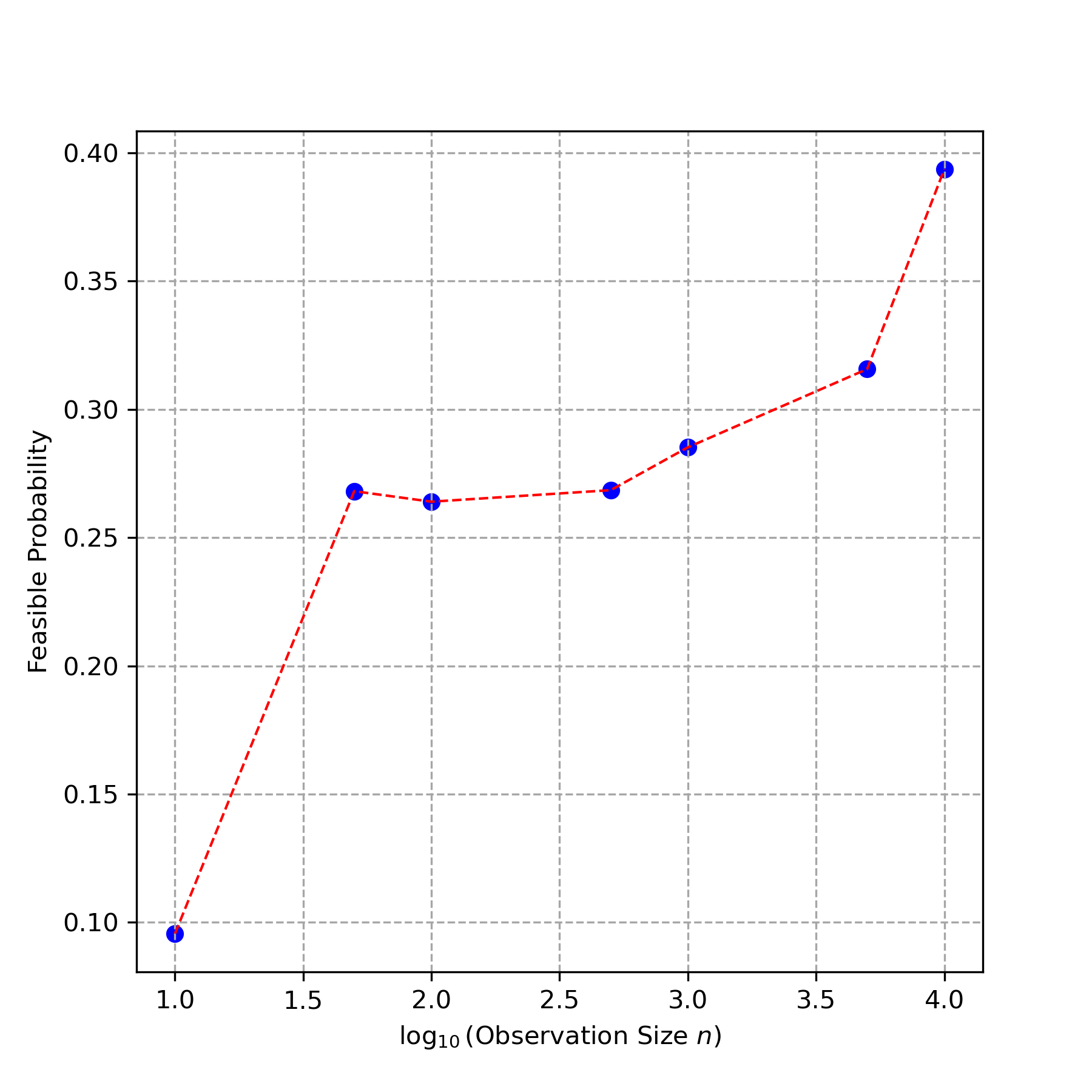}
         \caption{Feasible probability}
     \end{subfigure}
     \caption{(a) Euclidean distance to the closest feasible solution during the post-decoding process, given an infeasible generated solution. (b) The probability that a newly generated solution from the CVAE model is feasible. The x-axis represents the logarithm (base 10) of the observation size $n$.}
     \label{fig:ablation_2}
 \end{figure}

 \subsection{Debiasing for Fairness}
 \label{appdix:limitation}
 It is worth mentioning that this work shares the same concerns with prior research on latent-space Bayesian optimization using VAEs \cite{notin2021improving}, where the generated latent variables often fail to consistently represent the underlying data distributions. Such limitations could lead to social impacts similar to those observed with predictive policing systems \cite{meijer2019predictive}, which have faced significant controversy due to their reinforcement of existing biases.
 Specifically, redistricting to minimize workload variance based on historical 911 calls can inadvertently justify the concentration of policing efforts in districts with a history of higher incident rates.
 We recognize the critical importance of addressing these ethical concerns, particularly in the context of police redistricting based on historical data, which might inadvertently prioritize areas with historically higher incident rates.

 To mitigate these biases, our approach involves debiasing historical data using established machine learning techniques to ensure fairness in districting plans \cite{zhu2022spatiotemporal, jiang2024graph}. Furthermore, our \texttt{CageBO} algorithm is designed to accommodate implicit constraints aimed at promoting equitable districting, as suggested by practitioners. These measures align with our commitment to consider the broader impact of our work and the ethical implications of implementing machine learning solutions in public policymaking.

\end{document}